\documentclass{amsart}
\usepackage[utf8]{inputenc}
\usepackage[T1]{fontenc}
\usepackage[colorlinks=true,linkcolor=blue,urlcolor=blue,citecolor=blue,breaklinks]{hyperref}
\usepackage{amsfonts}
\usepackage{nicefrac}
\usepackage{microtype}
\usepackage{xcolor}
\usepackage{mathtools}
\usepackage{amsmath}
\usepackage{amsthm}
\usepackage{thmtools}
\usepackage{overpic}
\usepackage{multirow}
\usepackage{wrapfig}
\usepackage{enumitem}
\usepackage{amsaddr}
\usepackage{algorithm}
\usepackage{algpseudocodex}
\usepackage[numbers]{natbib}

\declaretheorem{theorem}

\declaretheorem[numberlike=theorem]{proposition}

\declaretheorem[numberlike=theorem]{lemma}
\declaretheorem[numberlike=theorem, style=remark]{remark}

\DeclareMathOperator*{\tr}{tr}
\DeclareMathOperator*{\argmin}{argmin}

\newcommand{\R}{\mathbb{R}}
\newcommand{\E}{\mathbb{E}}

\newcommand{\T}{\mathsf{T}}

\newcommand{\kl}[2]{\mathsf{KL}(#1\:\Vert\:#2)}

\renewcommand{\paragraph}[1]{\textbf{#1}}

\makeatletter
\def\paragraph{\@startsection{paragraph}{4}{\parindent}{10pt}{-5pt}{\normalsize\bf}}
\makeatother

\title[Probability Flow Solution of the Fokker-Planck Equation]{Probability Flow Solution of \\the Fokker-Planck Equation}

\author[Nicholas M. Boffi and Eric Vanden Eijnden]{Nicholas M.~Boffi and Eric Vanden Eijnden}
\address{Courant Institute of Mathematical Sciences\\
New York University, New York, NY 10012}
\email{boffi@cims.nyu.edu and eve2@cims.nyu.edu}

\begin{document}

\maketitle
\begin{abstract}The method of choice for integrating the time-dependent Fokker-Planck equation in high-dimension is to generate samples from the solution via integration of the associated stochastic differential equation. Here, we study an alternative scheme based on integrating an \textit{ordinary} differential equation that describes the flow of probability. 
Acting as a transport map, this equation deterministically pushes samples from the initial density onto samples from the solution at any later time. 
Unlike integration of the stochastic dynamics, the method has the advantage of giving direct access to quantities that are challenging to estimate from trajectories alone, such as the probability current, the density itself, and its entropy. 
The probability flow equation depends on the gradient of the logarithm of the solution (its ``score''), and so is \textit{a-priori} unknown. To resolve this dependence, we model the score with a deep neural network that is learned on-the-fly by propagating a set of samples according to the instantaneous probability current.
We show theoretically that the proposed approach controls the KL divergence from the learned solution to the target, while learning on external samples from the stochastic differential equation does not control either direction of the KL divergence.
Empirically, we consider several high-dimensional Fokker-Planck equations from the physics of interacting particle systems. We find that the method accurately matches analytical solutions when they are available as well as moments computed via Monte-Carlo when they are not. Moreover, the method offers compelling predictions for the global entropy production rate that out-perform those obtained from learning on stochastic trajectories, and can effectively capture non-equilibrium steady-state probability currents over long time intervals.\end{abstract}

\section{Introduction}
The time evolution of many dynamical processes occurring in the natural sciences, engineering, economics, and statistics are naturally described in the language of stochastic differential equations (SDE)~\citep{Gardiner,oksendal,evans2012introduction}. Typically, one is interested in the probability density function (PDF) of these processes, which describes the probability that the system will occupy a given state at a given time. The density can be obtained as the solution to a Fokker-Planck equation (FPE), which can generically be written as~\citep{risken1996fokker,bass2011stochastic}
\begin{equation}
\tag{FPE}
    \label{eq:fpe}
    \partial_t \rho^*_t(x) = - \nabla \cdot \left( b_t(x) \rho^*_t(x) - D_t(x) \nabla \rho^*_t(x)\right),  \qquad  x \in \Omega \subseteq \R^d,
\end{equation}
where $\rho^*_t(x) \in \R_{\geq 0}$ denotes the value of the density at time $t$, $b_t(x) \in \R^d $ is a vector field known as the drift, and $D_t(x) \in \R^{d\times d}$ is a positive-semidefinite tensor known as the diffusion matrix. \eqref{eq:fpe} must be solved for $t\ge0$ from some initial condition $\rho_{t=0}^*(x) = \rho_0(x)$, but in all but the simplest cases, the solution is not available analytically and can only be approximated via numerical integration. 

\paragraph{High-dimensionality.} For many systems of interest -- such as interacting particle systems in statistical physics~\citep{chandler1987introduction,spohn2012large}, stochastic control systems~\citep{kushner2001numerical}, and models in mathematical finance~\citep{oksendal} -- the dimensionality $d$ can be very large. This renders standard numerical methods for partial differential equations inapplicable, which become infeasible for $d$ as small as five or six due to an exponential scaling of the computational complexity with $d$. The standard solution to this problem is a Monte-Carlo approach, whereby the SDE associated with~\eqref{eq:fpe} 
\begin{equation}
    \label{eq:sde}
    dx_t = b_t(x_t) dt + \nabla \cdot D_t(x_t) dt + \sqrt{2} \sigma_t(x_t) dW_t,
\end{equation}
is evolved via numerical integration to obtain a large number $n$ of trajectories~\citep{kloeden1992stochastic}. In \eqref{eq:sde}, $\sigma_t(x)$ satisfies $\sigma_t(x)\sigma^\T_t(x) = D_t(x)$ and $W_t$ is a standard Brownian motion on $\R^d$. Assuming that we can draw samples $\{x_0^i\}_{i=1}^n$ from the initial PDF $\rho_0$, simulation of \eqref{eq:sde} enables the estimation of expectations via empirical averages
\begin{equation}
    \label{eq:expect:sde}
    \int_\Omega \phi(x) \rho^*_t(x) dx \approx \frac1n \sum_{i=1}^n \phi(x^i_t),
\end{equation}
where $\phi: \Omega \to \R$ is an observable of interest. While widely used, this method only provides samples from $\rho^*_t$, and hence other quantities of interest like the value of $\rho^*_t$ itself or the time-dependent differential entropy of the system $H_t = -\int_\Omega \log \rho^*_t(x) \rho^*_t (x) dx$ require sophisticated interpolation methods that typically do not scale well to high-dimension.

\paragraph{A transport map approach.} 
\label{sec:intro:transport_map}
Another possibility, building on recent theoretical advances that connect transportation of measures to the Fokker-Planck equation~\citep{jordan_variational_1998}, is to recast~\eqref{eq:fpe} as the transport equation~\citep{villani2009optimal,santambrogio2015optimal}
\begin{equation}
    \label{eq:transport:fpe}
    \partial_t \rho^*_t(x) = - \nabla \cdot \left( v^*_t(x) \rho^*_t(x)\right) 
\end{equation}
where we have defined the velocity field
\begin{equation}
\label{eq:velocity:prob}
    v^*_t(x) = b_t(x) - D_t(x) \nabla \log \rho^*_t(x).
\end{equation}
This formulation reveals that $\rho^*_t$ can be viewed as the pushforward of $\rho_0$ under the flow map $X_{\tau, t}^*(\cdot)$
of the ordinary differential equation
\begin{equation}
    \label{eq:probflow}
    \frac{d}{dt} X^*_{\tau,t}(x) = v^*_t(X^*_{\tau,t}(x)), \qquad X^*_{\tau,\tau}(x) = x, \quad t,\tau \ge 0.
\end{equation}
Equation~\eqref{eq:probflow} is known as the \textit{probability flow equation}, and its solution has the remarkable property that if $x$ is a sample from $\rho_0$, then $X_{0,t}^*(x)$ will be a sample from $\rho^*_t$. Viewing $X^*_{\tau, t}:\Omega \to \Omega$ as a transport map, $\rho^*_t = X^*_{0, t}\sharp\rho_0$ can be evaluated at any position in $\Omega$ via the change of variables formula \citep{villani2009optimal,santambrogio2015optimal} \begin{equation}
    \label{eq:rhot:rho0}
    \rho^*_t(x) = \rho_0(X^*_{t,0}(x)) \exp\left( - \int_0^t \nabla \cdot v^*_\tau(X^*_{t,\tau}(x)) d\tau\right)
\end{equation}
where $X^*_{t,0}(x)$ is obtained by solving \eqref{eq:probflow} backward from some given $x$. Importantly, access to the PDF as provided by~\eqref{eq:rhot:rho0} immediately gives the ability to compute quantities such as the probability current or the entropy; by contrast, this capability is absent when directly simulating the SDE.

\paragraph{Learning the flow.} The simplicity of the probability flow equation~\eqref{eq:probflow} is somewhat deceptive, because the velocity $v_t^*$ depends explicitly on the solution $\rho^*_t$ to the Fokker-Planck equation \eqref{eq:fpe}.
Nevertheless, recent work in generative modeling via score-based diffusion~\citep{ermon_song_1, ermon_song_2, song_how_2021} has shown that it is possible to use deep neural networks to estimate $v^*_t$, or equivalently the so-called \textit{score} $\nabla \log \rho^*_t$ of the solution density. 
Here, we introduce a variant of score-based diffusion modeling in which the score is learned on-the-fly over samples generated by the probability flow equation itself. 
The method is self-contained and enables us to bypass simulation of the SDE entirely; moreover, we provide both empirical and theoretical evidence that the resulting self-consistent training procedure offers improved performance when compared to training via samples produced from simulation of the SDE.

\subsection{Contributions}
Our contributions are both theoretical and computational:
\begin{itemize}[leftmargin=.2in]
    \item We provide a bound on the Kullback-Leibler divergence from the estimate $\rho_t$ produced via an approximate velocity field $v_t$ to the target $\rho_t^*$. This bound motivates our approach, and shows that minimizing the discrepancy between the learned score and the score of the push-forward distribution systematically improves the accuracy of $\rho_t$.
    \item Based on this bound, we introduce two optimization problems that can be used to learn the velocity field~\eqref{eq:velocity:prob} in the transport equation~\eqref{eq:transport:fpe} so that its solution coincides with that of the Fokker Planck equation~\eqref{eq:fpe}. 
    Due to its similarities with score-based diffusion approaches in generative modeling (SBDM), we call the resulting method \textit{score-based transport modeling} (SBTM).
    \item We provide specific estimators for quantities that can be computed via SBTM but are not directly available from samples alone, like point-wise evaluation of $\rho_t$ itself, the differential entropy, and the probability current. 
    \item We test SBTM on several examples involving interacting particles that pairwise repel but are kept close by common attraction to a moving trap. In these systems, the FPE is high-dimensional due to the large number of particles, which vary from 5 to 50 in the examples below. Problems of this type frequently appear in the molecular dynamics of externally-driven soft matter systems~\citep{frenkel2001understanding,spohn2012large}. We show that our method can be used to accurately compute the entropy production rate, a quantity of interest in the active matter community~\citep{nardini2017entropy}, as it quantifies the out-of-equilibrium nature of the system's dynamics.
\end{itemize}

\subsection{Notation and assumptions.}
\label{sec:assumptions}
Throughout, we assume that the stochastic process \eqref{eq:sde} evolves over a domain $\Omega\subseteq \R^d$ in which it remains at all times~$t\ge0$. We assume that the drift vector $b_t: \Omega \to \R^d$ and the diffusion tensor $D_t: \Omega \to \R^{d\times d}$ are  twice-differentiable and bounded in both $x$ and $t$, so that the solution to the SDE~\eqref{eq:sde} is well-defined at all times $t\ge0$. The symmetric tensor $D_t(x) = D_t^\T(x)$ is assumed to be positive semi-definite for each $(t,x)$, with Cholesky decomposition $D_t(x) = \sigma_t(x) \sigma^\T_t(x)$. We further assume that the initial PDF $\rho_0$ is three-times differentiable, positive everywhere on $\Omega$, and such that $H_0 = -\int_\Omega \log \rho_0(x) \rho_0(x) dx < \infty$. This guarantees that $\rho^*_t$ enjoys the same properties at all times $t>0$. Finally, we assume that $\log \rho^*_t$ is $K$-smooth globally for $(t,x)\in [0,\infty)\times \Omega$, i.e. 
\begin{equation}
    \label{eq:Ksmooth}
    \exists K>0 \ : \quad \forall (t,x)\in [0,\infty)\times \Omega \quad |\nabla \log \rho^*_t(x) - \nabla \log \rho^*_t(y)| \le K |x-y|.
\end{equation}
This technical assumption is needed to guarantee global existence and uniqueness of the solution of the probability flow equation. Throughout, we use the shorthand notation $\dot{y}_t = \frac{d}{dt}y_t$ interchangeably for a time-dependent quantity $y_t$.

\section{Related work}
\paragraph{Score matching} Our approach builds directly on the toolbox of score matching originally developed by Hyv{\"a}rinen~\citep{hyvarinen_estimation_2005, hyvarinen_connections_2007, hyvarinen_extensions_2007, hyvarinen_optimal_2008} and more recently extended in the context of diffusion-based generative modeling~\citep{ermon_song_1, ermon_song_2, song_SDE, de_bortoli_diffusion_2021, dockhorn_score-based_2022, mittal_symbolic_2021}.
These approaches assume access to training samples from the target distribution (e.g., in the form of examples of natural images).
Here, we bypass this need and use the probability flow equation to obtain the samples needed to learn an approximation of the score.
\citet{lu2021maximum} recently showed that using the transport equation~\eqref{eq:transport} with a velocity field learned via SBDM can lead to inaccuracies in the likelihood unless higher-order score terms are well-approximated. Proposition~\ref{prop:entropy} shows that the self-consistent approach used in SBTM solves these issues and ensures a systematic approximation of the target $\rho_t^*$.
~\citet{lai_improving_2023} recently used a similar idea to improve sample quality with score-based probability flow equations in generative modeling.

\paragraph{Density estimation and Bayesian inference} Our method shares commonalities with transport map-based approaches~\citep{marzouk_introduction_2016} for density estimation and variational inference~\citep{zhang_advances_2019, blei_variational_2017} such as normalizing flows~\citep{tabak_density_2010, tabak_family_2013, rezende_variational_2016, huang_convex_2021, papamakarios_normalizing_2021, kobyzev_normalizing_2021}. Moreover, because expectations are approximated over a set of samples according to~\eqref{eq:expect:sde}, the method also inherits elements of classical ``particle-based'' approaches for density estimation such as Markov chain Monte Carlo~\citep{mcmc} and sequential Monte Carlo~\citep{dai_invitation_2020, del_moral_sequential_2006}.

Our approach is also reminiscent of a recent line of work in Bayesian inference that aims to combine the strengths of particle methods with those of variational approximations~\citep{dai_provable_2016, saeedi_variational_2017}. In particular, the method we propose bears some similarity with Stein variational gradient descent (SVGD)~\citep{liu_stein_2017, liu_stein_2018, liu_stein_2019} (see also~\citep{lu_scaling_2018, li_stochastic_2020}), in that both methods approximate the target distribution via \textit{deterministic} propagation of a set of samples. The key differences are that (i) our method learns the map used to propagate the samples, while the map in SVGD corresponds to optimization of the kernelized Stein discrepancy, and (ii) the methods have distinct goals, as we are interested in capturing the dynamical evolution of $\rho_t^*$ rather than sampling from an equilibrium density. Indeed, many of the examples we consider do not have an equilibrium density, i.e. $\lim_{t\rightarrow\infty} \rho_t^*$ does not exist.

\paragraph{Approaches for solving the FPE} Most closely connected to our paper are the  works by~\citet{maoutsa_interacting_2020} and~\citet{shen2022self}, who similarly propose to bypass the SDE through use of the probability flow equation, building on earlier work by~\citet{degond1990determinsitic} and~\citet{russo_deterministic_1990}.
The critical differences between~\citet{maoutsa_interacting_2020} and our approach are that they perform estimation over a linear space or a reproducing kernel Hilbert space rather than over the significantly richer class of neural networks, and that they train using the original score matching loss of~\citet{hyvarinen_estimation_2005}, while the use of neural networks requires the introduction of regularized variants.
Because of this,~\cite{maoutsa_interacting_2020} studies systems of dimension less than or equal to five; in contrast, we study systems with dimensionality as high as $100$.

Concurrently to our work,~\citet{shen2022self} proposed a variational problem similar to SBTM. A key difference is that SBTM is not limited to Fokker-Planck equations that can be viewed as a gradient flow in the Wasserstein metric over some energy (i.e., the drift term in the SDE~\eqref{eq:sde} need not be the gradient of a potential), and that it allows for spatially-dependent and rank-deficient diffusion matrices. Moreover, our theoretical results are similar, but by avoiding the use of costly Sobolev norms lead to a practical optimization problem that we show can be solved in high dimension and over long times. In a follow-up to~\citet{shen2022self} and our present work,~\citet{li_self-consistent_2023} propose an algorithm that can be seen as an expectation-maximization algorithm for the loss function in~\eqref{eq:sbtm}, which avoids calculation of $G_t$ according to equation~\eqref{eq:XtGt}.

\paragraph{Neural-network solutions to PDEs} Our approach can also be viewed as an alternative to recent neural network-based methods for the solution of partial differential equations (see e.g.~\cite{e2017deep,Raissi2019physics,han2018solving,sirigano2018dgm,bruna2022neural}). Unlike these existing approaches, our method is tailored to the solution of the Fokker-Planck equation and guarantees that the solution is a valid probability density. Our approach is fundamentally Lagrangian in nature, which has the advantage that it only involves learning quantities locally at the positions of a set of evolving samples; this is naturally conducive to efficient scaling for high-dimensional systems.

\section{Methodology}
\subsection{Score-based transport modeling} 
\label{sec:sbtm}
Let $s_t:\Omega \to\R^d$ denote an approximation to the score of the target $\nabla\log\rho_t^*$, and consider the solution $\rho_t: \Omega \rightarrow \R_{\geq 0}$ to the transport equation
\begin{equation}
    \label{eq:transport}
    \tag{TE}
    \partial_t\rho_t(x) = - \nabla \cdot (v_t(x)\rho_t(x)) \qquad \text{with} \quad v_t(x) = b_t(x) - D_t(x) s_t(x).
\end{equation}
Our goal is to develop a variational principle that may be used to adjust $s_t$ so that~$\rho_t$ tracks~$\rho_t^*$. Our approach is based on the following inequality, whose proof may be found in Appendix~\ref{app:ent:prod}:
\begin{restatable}[Control of the KL divergence]{proposition}{sbtment}
\label{prop:entropy}
Assume that the conditions listed in Sec.~\ref{sec:assumptions} hold. Let $\rho_t$ denote the solution to the transport equation~\eqref{eq:transport}, and let $\rho^*_t$ denote the solution to the Fokker-Planck equation~\eqref{eq:fpe}. Assume that $\rho_{t=0}(x) = \rho^*_{t=0}(x) = \rho_0(x)$ for all $x \in \Omega$. Then 
\begin{equation}
    \label{eq:ent:prod:rate:1}
     \frac{d}{dt} \kl{\rho_t}{\rho^*_t} \le \frac{1}2  \int_\Omega\left|s_t(x) - \nabla \log \rho_t(x)\right|_{D_t(x)}^2 \rho_t(x)dx,
\end{equation}
where $|\cdot|^2_{D_t(x)} = \langle \cdot, D_t(x) \cdot \rangle$.
\end{restatable}
In particular,~\eqref{eq:ent:prod:rate:1} implies that for any $T\in[0,\infty)$ we have explicit control on the KL divergence
\begin{equation}
    \label{eq:ent:prod:n}
     \kl{\rho_T}{\rho^*_T} \le \frac{1}2  \int_0^T \int_\Omega\left|s_t(x) - \nabla \log \rho_t(x)\right|_{D_t(x)}^2 \rho_t(x)dx dt.
\end{equation}
Remarkably,~\eqref{eq:ent:prod:n} only depends on the approximate $\rho_t$ and does not include~$\rho_t^*$: it states that the accuracy of $\rho_t$ as an approximation of $\rho_t^*$ can be improved by enforcing agreement between $s_t$ and $\nabla\log\rho_t$. This means that we can optimize~\eqref{eq:ent:prod:n} \textit{without} making use of external data from $\rho_t^*$, which offers a self-consistent objective function to learn the score $s_t$ using~\eqref{eq:transport} alone.

The primary difficulty with this approach is that $\rho_t$ must be considered as a  functional of~$s_t$, since the velocity~$v_t$ used in~\eqref{eq:transport} depends on~$s_t$. To render the resulting minimization of the right-hand side of~\eqref{eq:ent:prod:n} practical, we can exploit that~\eqref{eq:transport} can be solved via the method of characteristics, as summarized in Appendix~\ref{app:basic}. Specifically, if $\dot{X}_t(x) = v_t(X_t(x))$ is the probability flow equation associated with the velocity $v_t$, then $\rho_t = X_t \sharp \rho_0$. This means that the expectation of any function $\phi(x)$ over $\rho_t(x)$ can be expressed as the expectation of $\phi_t(X_t(x))$ over $\rho_0(x)$. Observing that the score of the solution to~\eqref{eq:transport} along trajectories of the probability flow $\nabla\log\rho_t(X_t(x))$ solves a closed equation leads to the following proposition.
\begin{restatable}[Score-based transport modeling]{proposition}{sbtmtwo}
\label{prop:loss2} Assume that the conditions listed in Sec.~\ref{sec:assumptions} hold.  Define $v_t(x) = b_t(x) - D_t(x) s_t(x)$ and consider 
\begin{equation}
\label{eq:XtGt}
\begin{aligned}
&\dot X_{t}(x) = v_t(X_{t}(x)), && X_{0}(x) = x,\\
&\dot G_t(x) = - [\nabla  v_t(X_{t}(x))]^\T G_t(x) - \nabla \nabla \cdot v_t(X_{t}(x)), \quad &&G_0(x) = \nabla \log \rho_0(x).
\end{aligned}
\end{equation}
Then $\rho_t = X_t\sharp \rho_0$ solves~\eqref{eq:transport}, the equality $G_t(x) = \nabla \log \rho_t(X_t(x))$ holds, and for any $T\in[0,\infty)$
\begin{equation}
    \label{eq:ent:prod}
     \kl{X_T\sharp\rho_0}{\rho^*_T} \le \frac{1}2  \int_0^T \int_\Omega\left|s_t(X_t(x)) - G_t(x)\right|_{D_t(X_t(x))}^2 \rho_0(x)dx dt.
\end{equation}
Moreover, if $s^*_t$ is a minimizer of the constrained optimization problem
\begin{equation}
\label{eq:sbtm}\tag{SBTM}
\min_{s} \int_0^T \int_\Omega \left|s_t(X_t(x)) - G_t(x)\right|_{D_t(X_t(x))}^2 \rho_0(x) dx dt \quad \text{subject to }\eqref{eq:XtGt}
\end{equation}
then $D_t(x) s^*_t(x) = D_t(x)\nabla \log \rho_t^*(x)$  where $\rho^*_t$ solves the Fokker-Planck equation~\eqref{eq:fpe}. The map $X^*_t$ associated to any minimizer is a transport map from $\rho_0$ to $\rho^*_t$, i.e.
\begin{equation}
    \label{eq:pdf:prob:flow}
    x \sim \rho_0 \qquad \text{implies that} \qquad X^*_{t}(x) \sim \rho^*_t, \qquad \forall t\in[0,T].
\end{equation}
\end{restatable}
Proposition~\ref{prop:loss2} is proven in Appendix~\ref{app:sbtm:lag}. The result also holds with a standard Euclidean norm replacing the diffusion-weighted norm, in which case the minimizer is unique and is given by $s^*_t(x) = \nabla \log \rho_t^*(x)$. In the special case when the SDE is an Ornstein-Uhlenbeck process, the score and the equations for both $X_t$ and $G_t$ can be written explicitly; they are studied in Appendix~\ref{app:gauss}.

In practice, the objective in~\eqref{eq:sbtm} can be estimated empirically by generating samples from $\rho_0$ and solving the equations for $X_t(x)$ and $G_t(x)$ with $x\sim\rho_0$. The constrained minimization problem~\eqref{eq:sbtm} can then in principle be solved with gradient-based techniques via the adjoint method. The corresponding equations are written in Appendix~\ref{app:sbtm:lag}, but they involve fourth-order spatial derivatives that are computationally expensive to compute via automatic differentiation. Moreover, each gradient step requires solving a system of ordinary differential equations whose dimensionality is equal to the number of samples used to compute expectations times the dimension of~\eqref{eq:fpe}. Instead, we now develop a sequential timestepping procedure that avoids these difficulties entirely, and as a byproduct can scale to arbitrarily long time windows.

\subsection{Sequential score-based transport modeling}
\label{sec:local}
An alternative to the constrained minimization in Proposition~\ref{prop:loss2} is to consider an approach whereby the score $s_t$ is obtained independently at each time to ensure that $\kl{\rho_t}{\rho_t^*}$ remains small.
This suggests choosing $s_t$ to minimize $\frac{d}{dt}\kl{\rho_t}{\rho_t^*}$, which admits a simple closed-form bound, as shown in Proposition~\ref{prop:entropy}.
While this explicit form can be used directly, an application of Stein's identity recovers an implicit objective analogous to Hyv\"arinen score-matching that is equivalent to minimizing $\frac{d}{dt}\kl{\rho_t}{\rho_t^*}$ but obviates the calculation of $G_t$. Expanding the square in~\eqref{eq:ent:prod:rate:1} and applying $\int_\Omega s_t(x)^\T \nabla \log \rho_t(x) \,\rho_t(x)dx = - \int_\Omega \nabla\cdot s_t(x) \,\rho_t(x)dx$, we may write
\begin{equation*}
    \label{eq:ent:prod:rate:2}
    \begin{aligned}
     \frac{d}{dt}\kl{\rho_t}{\rho^*_t} &\le \frac{1}{2}  \int_\Omega\big(|s_t(X_t(x))|_{D_t(X_t(x))}^2 + 2\nabla \cdot (D_t(X_t(x))s_t(X_t(x)))\big)\rho_0(x)dx\\
     &\qquad + \frac{1}{2}\int |G_t(x)|^2 \rho_0(x)dx.
      \end{aligned}
\end{equation*}
Because $\nabla\log\rho_t(X_t(x)) = G_t(x)$ is independent of $s_t$, we may neglect the corresponding square term during optimization.
This leads to a simple and comparatively less expensive way to build the pushforward $X^*_t$ such that $X^*_t\sharp \rho_0 = \rho_t^*$ sequentially in time, as stated in the following proposition.
\begin{restatable}[Sequential SBTM]{proposition}{sbtmthree}
\label{prop:loss3}
In the same setting as Proposition~\ref{prop:loss2}, let $X_t(x)$  solve the first equation in~\eqref{eq:XtGt} with $v_t(x) = b_t(x) - D_t(x)s_t(x)$. Let $s_t$ be obtained via
\begin{equation}
    \label{eq:sbtm3}
    \tag{SSBTM}
    \min_{s_t} \int_\Omega \left(|s_t(X_t(x))|_{D_t(X_t(x))}^2 + 2 \nabla \cdot ( D_t(X_t(x))s_t(X_t(x))) \right) \rho_0(x) dx.
\end{equation}
Then, each minimizer $s_t^*$ of \eqref{eq:sbtm3} satisfies $D_t(x) s^*_t(x) = D_t(x) \nabla \log \rho^*_t(x)$ where $\rho^*_t$ is the solution to~\eqref{eq:fpe}. Moreover, the map $X^*_t$ associated to $s_t^*$ is a transport map from $\rho_0$ to $\rho^*_t$. 
\end{restatable}
Proposition~\ref{prop:loss3} is proven in Appendix~\ref{app:sbtm:seq}. Critically,~\eqref{eq:sbtm3} is no longer a constrained optimization problem. Given the current value of $X_t$ at any time $t$, we can obtain $s_t$ via direct minimization of the objective in~\eqref{eq:sbtm3}. Given $s_t$, we may compute the right-hand side of~\eqref{eq:XtGt} and propagate $X_t$ (and possibly $G_t$) forward in time. The resulting procedure, which alternates between self-consistent score estimation and sample propagation, is presented in Algorithm~\ref{alg:sbtm} for the choice of a forward-Euler integration routine in time. The output of the method produces a feasible solution for~\eqref{eq:sbtm} with an \textit{a-posteriori} bound on the loss obtained via integration. A few remarks on Algorithm~\ref{alg:sbtm} are now in order.

\begin{algorithm}[t]
    \caption{Sequential score-based transport modeling.}
    \label{alg:sbtm}
    \begin{algorithmic}[1]
        \State \textbf{Input}: An initial time $t_0 \in \R_{\geq 0}$. A set of $n$ samples $\{x_i\}_{i=1}^n$ from $\rho_{t_0}$. A set of $N_T$ timesteps $\{\Delta t_k\}_{k=0}^{N_T-1}$.
        \State Initialize sample locations $X^i_{t_0} = x_i$ for $i = 1, \hdots, n$.
        \For{$k=0, \hdots, N_t-1$}
            \State Optimize: $s_{t_k} = \argmin_{s} \frac{1}{n} \sum_{i=1}^n \left[|s(X^i_{t_k})|_{D_{t_k}(X^i_{t_k})}^2 + 2 \nabla\cdot\left(D_{t_k}(X^i_{t_k}) s(X_{t_k}^i)\right)\right]$. 
            \State Propagate samples: \[\quad X^i_{t_{k+1}} = X^i_{t_{k}} + \Delta t_{k} \left(b_{t_{k}}(X^i_{t_k}) - D_{t_{k}}(X_{t_k}^i)s_{t_{k}}(X_{t_k}^i)\right).\] 
            \State Set $t_{k+1} = t_{k} + \Delta t_k$.
        \EndFor
        \State \textbf{Output}: A set of $n$ samples $\{X_{t_k}^i\}_{i=1}^n$ from $\rho_{t_k}$ and the score $\{s_{t_k}(X^i_{t_k})\}_{i=1}^n$ for all $\{t_k\}_{k=0}^{N_T}$.
    \end{algorithmic}
\end{algorithm}

\paragraph{Higher-order integrators}
Algorithm~\ref{alg:sbtm} is stated for choice of forward-Euler integration for simplicity. In practice, any off-the-shelf integrator can be used, such as an adaptive Runge-Kutta method, by temporal discretization of the dynamics
\begin{equation*}
    \begin{aligned}
        \dot{X}_t(x) &= v_t(X_t(x)) \\
        s_t &= \argmin_{s} \int_\Omega \left(|s(X_t(x))|_{D_t(X_t(x))}^2 + 2 \nabla \cdot ( D_t(X_t(x))s(X_t(x))) \right) \rho_0(x) dx.
    \end{aligned}
\end{equation*}
and spatial discretization of the expectation over a set of samples propagated according to the equation for $X_t(x)$. In practice, the minimization can be performed over a parametric class of functions such as neural networks via a few steps of gradient descent.


\paragraph{Divergence computation} To avoid computation of the divergence -- which can be costly for neural networks with high input dimension -- we can use the denoising score matching loss function introduced by~\cite{vincent_connection_2011}, which we discuss in Appendix~\ref{app:denoise}. Empirically, we find that use of either the denoising objective or explicit derivative regularization is necessary for stable training to avoid overfitting to the training data; the level of regularization (or the noise scale in the denoising objective) can be decreased as the size of the dataset increases. 

\paragraph{Time-dependence} When optimizing over a parametric class of functions, the score can be taken to be explicitly time-dependent, or the time-dependence can originate only through the parameters. In either case, all required outputs can be computed on-the-fly to avoid saving the entire history of parameters, which could be memory-intensive for large neural networks. If a time-dependent architecture is used, the method is amenable to online learning by randomly re-drawing initial conditions and optimizing over the resulting trajectory. In the numerical experiments below, we consider time-independent models with time-dependent parameters, because we found them to be sufficient.

\paragraph{SBTM vs. Sequential SBTM} Given the simplicity of the optimization problem~\eqref{eq:sbtm3}, one may wonder if~\eqref{eq:sbtm} is useful in practice, or if it is simply a stepping stone to arrive at~\eqref{eq:sbtm3}.
The primary difference is that~\eqref{eq:sbtm} offers global control on the discrepancy between $s_t$ and $\nabla\log\rho_t$ over $t\in[0,T]$ that unavoidably arises in practice due to learning and time-discretization errors. By contrast, because~\eqref{eq:sbtm3} proceeds sequentially, these errors could accumulate over time in a way that is harder to control.
In the numerical examples below, we took the timestep $\Delta t$ sufficiently small, and the number of samples $n$ sufficiently large, that we did not observe any accumulation of error. 
Nevertheless,~\eqref{eq:sbtm} may allow for more accurate approximation, because the loss is exactly minimized at zero and high-order derivatives of $s_t$ are controlled through calculation of $\dot{G}_t$.

\paragraph{Why not train on external data?} An alternative to the sequential procedure outlined here would be to generate samples from the target $\rho_t^*$ via simulation of the associated SDE, and to approximate the score $\nabla\log\rho_t^*$ via minimization of the loss $\int_0^T \int_{\Omega} (|s_t(x)|^2 + 2\nabla \cdot s_t(x)) \rho_t^*(x)dxdt$, similar to SBDM. As shown in Appendix~\ref{app:sde_learn} neither $\kl{\rho_t}{\rho_t^*}$ nor $\kl{\rho_t^*}{\rho_t}$ are controlled when using this procedure, where $\rho_t = X_t\sharp\rho_0$ is the density of the probability flow equation. Empirically, we find in the numerical experiments that this approach is significantly less stable than sequential SBTM. In particular, and importantly for the applications we consider, we could not stably estimate the trajectory of the entropy production rate using a score model learned from the SDE with the same number of samples as used for SBTM.

\section{Numerical experiments}
In the following, we study two high-dimensional examples from the physics of interacting particle systems, where the spatial variable of the Fokker-Planck equation~\eqref{eq:fpe} can be written as $x = \left(x^{(1)}, x^{(2)}, \hdots, x^{(N)}\right)^\T$ with each $x^{(i)} \in \R^{\bar{d}}$. Here, $\bar{d}$ describes a lower-dimensional ambient space, e.g. $\bar{d} = 2$,  so that the dimensionality of the Fokker-Planck equation $d = N \bar{d}$ will be high if the number of particles $N$ is even moderate\footnote{We would like to emphasize at this stage the difference between the number of \textit{physical particles} $N$, which is a parameter for the system under study and sets the dimensionality of the resulting FPE, and the number of \textit{algorithmic samples} $n$, which is a hyper-parameter that can be chosen at will to improve the accuracy of the learning.}. The still figures shown in this section do not fully depict the complexity of the interacting particle dynamics, and we encourage the reader to view the movies available \href{https://drive.google.com/drive/folders/1JS9H7f9G0JnMe8_Fin22bf6bl3Drvid9?usp=share_link}{here}. With a timestep $\Delta t = 10^{-3}$, a horizon $T=10$, and a fixed $n N \bar{d} = 10^5$, we find that the sequential SBTM procedure takes around two hours for each simulation on a single NVIDIA RTX8000 GPU. In addition, we conclude with a low-dimensional example from the physics of active matter, which highlights the ability of sequential SBTM to remain stable over long times and to capture non-equilibrium probability currents.

\subsection{Harmonically interacting particles in a harmonic trap}
\label{sec:harmonic}
\begin{figure}[t]
\begin{tabular}{c}
    \begin{overpic}[width=\textwidth]{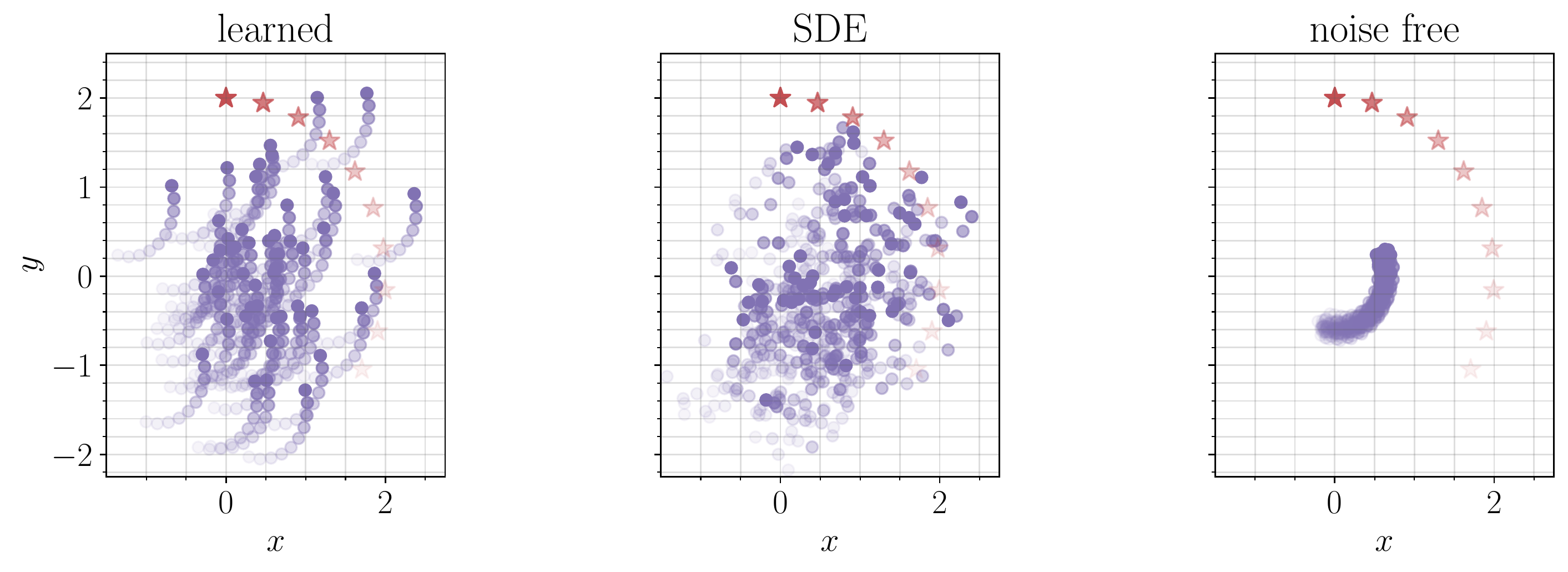}
    \put(0, 35){\textbf{A}}
    \end{overpic} \\
    \begin{overpic}[width=\textwidth]{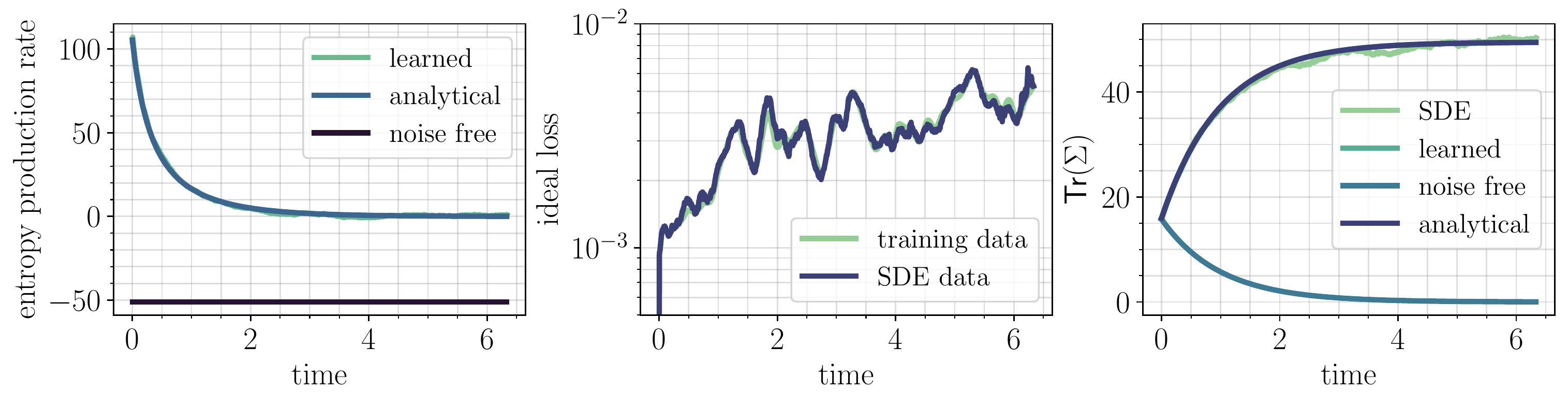}
    \put(0, 27){\textbf{B}}
    \end{overpic}
\end{tabular}
    \caption{\textit{A system of $N=50$ particles in a harmonic trap with a harmonic interaction:} (A) A single sample trajectory. The mean of the trap $\beta_t$ is shown with a red star, while past positions of the particles are indicated by a fading trajectory. The noise-free system (right) is too concentrated, and fails to capture the variance of the stochastic dynamics (center). The learned system (left) accurately captures the variance, and in addition generates physically interpretable trajectories for the particles. (B) Quantitative comparison to the analytical solution. The learned solution matches the entropy production rate, score, and covariance well. A movie of the particle motion can be found \href{https://drive.google.com/file/d/17AErES6VEacEcXv50sebpt7Tc7OsY8YW/view?usp=share_link}{here}.}
    \label{fig:harmonic_particles}
\end{figure}

\paragraph{Setup.} 
Here we study a problem that admits a tractable analytical solution for direct comparison. We consider $N$ two-dimensional particles ($\bar{d} = 2$) that repel according to a harmonic interaction but experience harmonic attraction towards a moving trap $\beta_t \in \R^2$. The motion of the physical particles is governed by the stochastic dynamics
\begin{equation}
    \label{eqn:harmonic_SDE}
    dX^{(i)}_t = (\beta_t - X^{(i)}_t)dt  + \alpha \Big(X^{(i)}_t - \frac{1}{N}\sum_{j=1}^N X^{(j)}_t\Big)dt + \sqrt{2D}\,dW_t^{(i)},\quad i=1, \hdots, N
\end{equation}
where $\alpha \in (0,1)$ is a fixed coefficient that sets the magnitude of the repulsion. The dynamics~\eqref{eqn:harmonic_SDE} is an Ornstein-Uhlenbeck process in the extended variable $x \in \R^{\bar{d}N}$ with block components $x^{(i)}$. Assuming a Gaussian initial condition, the solution to the  Fokker-Planck equation associated with~\eqref{eqn:harmonic_SDE} is a  Gaussian for all time and hence can be characterized entirely by its mean $m_t$ and covariance $C_t$. These can be obtained analytically (Appendices~\ref{app:gauss} and~\ref{app:exp}), which facilitates a quantitative comparison to the learned model. The differential entropy $S_t$ is given by
\begin{equation}
    \label{eq:St:harm}
    H_t  = \tfrac12 \bar d N\left(\log\left(2\pi\right) + 1\right) + \tfrac{1}{2}\log\det C_t.
\end{equation}

In the experiments, we take $\beta_t = a(\cos\pi\omega t, \sin\pi\omega t)^\T$ with $a = 2$, $\omega = 1$, $D = 0.25$, $\alpha = 0.5$, and $N=50$, giving rise to a $100$-dimensional Fokker-Planck equation. The particles are initialized from an isotropic Gaussian with mean $\beta_0$ (the initial trap position) and variance $\sigma_0^2 = 0.25$.

\paragraph{Network architecture.} We take $s_t(x) = -\nabla U_{\theta_t}(x)$, where the potential $U_{\theta_t}(\cdot)$ is given as a sum of one- and two-particle terms
\begin{equation}
    \label{eqn:parametric_potential}
    U_{\theta_t}\big(x^{(1)}, \hdots, x^{(N)}\big) = \sum_{i=1}^N U_{\theta_t, 1}\big(x^{(i)}\big) + \frac{1}{N}\sum_{\substack{i,j=1\\ i \neq j}}^N U_{\theta_t, 2}\big(x^{(i)}, x^{(j)}\big),
\end{equation}
which ensures permutation symmetry amongst the physical particles by direct summation over all pairs. Modeling at the level of the potential introduces an additional gradient into the loss function, but makes it simple to enforce permutation symmetry; moreover, by writing the potential as a sum of one- and two-particle terms, the dimensionality of the function estimation problem is reduced. As motivation for this choice of architecture, we show in Appendix~\ref{app:exp:harmonic} that the class of scores representable by~\eqref{eqn:parametric_potential} contains the analytical score for the harmonic problem considered in this section. To obtain the parameters $\theta_{t_k+\Delta t_k}$, we perform a warm start and initialize from $\theta_{t_k}$, which reduces the number of optimization steps that need to be performed at each iteration. All networks are taken to be multi-layer perceptrons with the $\texttt{swish}$ activation function~\citep{swish_act}; further details on the architectures used can be found in Appendix~\ref{app:exp}.

\paragraph{Quantitative comparison.} For a quantitative comparison between the learned model and the exact solution, we study the empirical covariance $\Sigma$ over the samples and the entropy production rate $\frac{dS_t}{dt}$. Because an analytical solution is available for this system, we may also compute the target $\nabla \log \rho_t(x) = -C_t^{-1}(x - m_t)$ and measure the goodness of fit via the relative Fisher divergence 
\begin{equation}
    \label{eq:discrep}
    \frac{\int_\Omega |s_t(x) - \nabla\log\rho_t(x)|^2 \bar{\rho}(x) dx}{\int_\Omega|\nabla\log\rho_t(x)|^2 \bar{\rho}(x)dx}.
\end{equation} 
In Equation~\eqref{eq:discrep}, $\bar{\rho}$ can be taken to be equal to the current empirical estimate of $\rho_t$ (the training data), or estimated using samples from the stochastic differential equation (the SDE data).%

\paragraph{Results.}
The representation of the dynamics~\eqref{eqn:harmonic_SDE} in terms of the flow of probability leads to an intuitive deterministic motion that accurately captures the statistics of the underlying stochastic process. Snapshots of particle trajectories from the learned probability flow \eqref{eq:probflow}, the SDE \eqref{eqn:harmonic_SDE}, and the noise-free equation obtained by setting $D = 0$ in \eqref{eqn:harmonic_SDE}
are shown in Figure~\ref{fig:harmonic_particles}A. 

Results for this quantitative comparison are shown in Figure~\ref{fig:harmonic_particles}B. The learned model accurately predicts the entropy production rate of the system and minimizes the relative metric~\eqref{eq:discrep} to the order of $10^{-2}$. The noise-free system incorrectly predicts a constant and negative entropy production rate, while the SDE cannot make a prediction for the entropy production rate without an additional learning component; we study this possibility in the next example. In addition, the learned model accurately predicts the high-dimensional covariance $\Sigma$ of the system (curves lie directly on top of the analytical result, trace shown for simplicity). The SDE also captures the covariance, but exhibits more fluctuations in the estimate; the noise-free system incorrectly estimates all covariance components as decaying to zero.

\subsection{Soft spheres in an anharmonic trap}
\label{sec:anharmonic_N5}
\begin{figure}[t!]
\begin{tabular}{cc}
    \multicolumn{2}{l}{%
    \begin{overpic}[width=\textwidth]{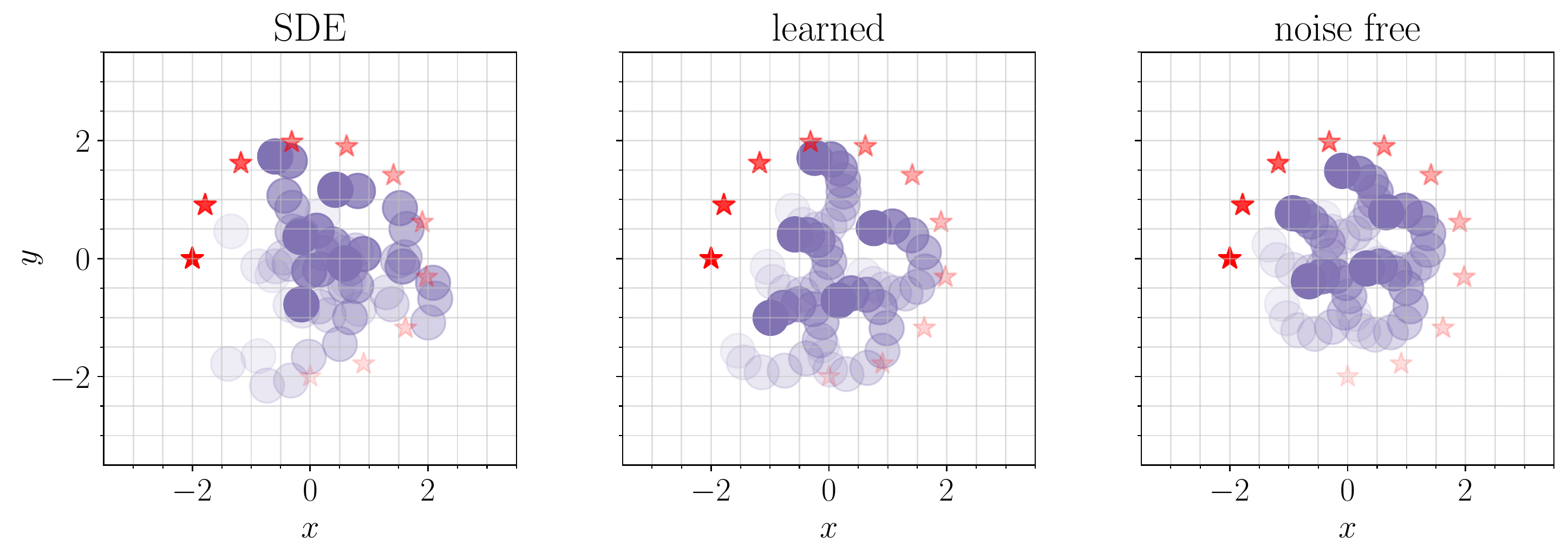}
    \put(0, 35){\textbf{A}}
    \end{overpic}
    }\\
    \begin{overpic}[width=.5\textwidth]{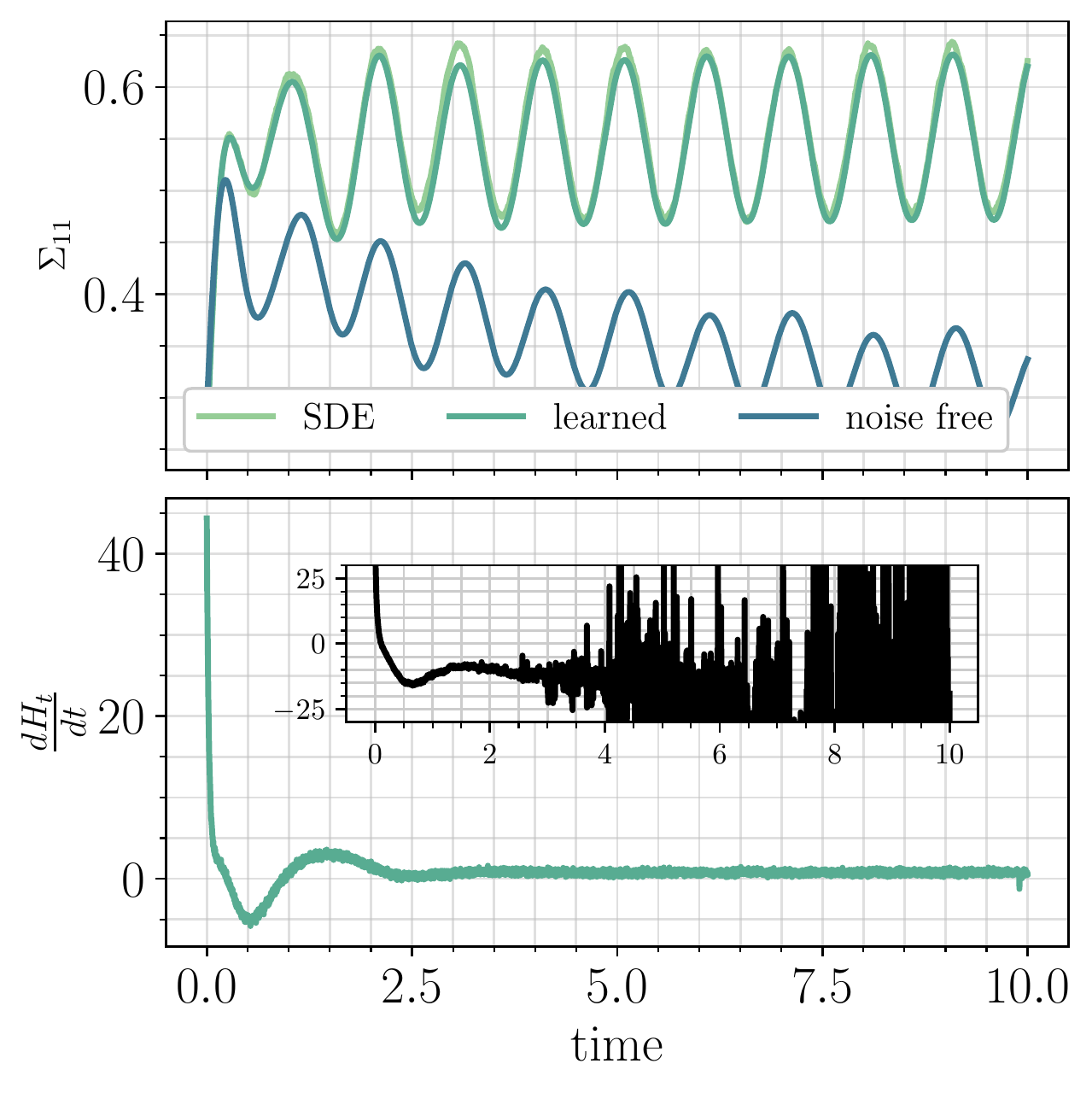}
    \put(0, 98){\textbf{B}}
    \put(0, 50){\textbf{D}}
    \end{overpic}&
    \begin{overpic}[width=.5\textwidth]{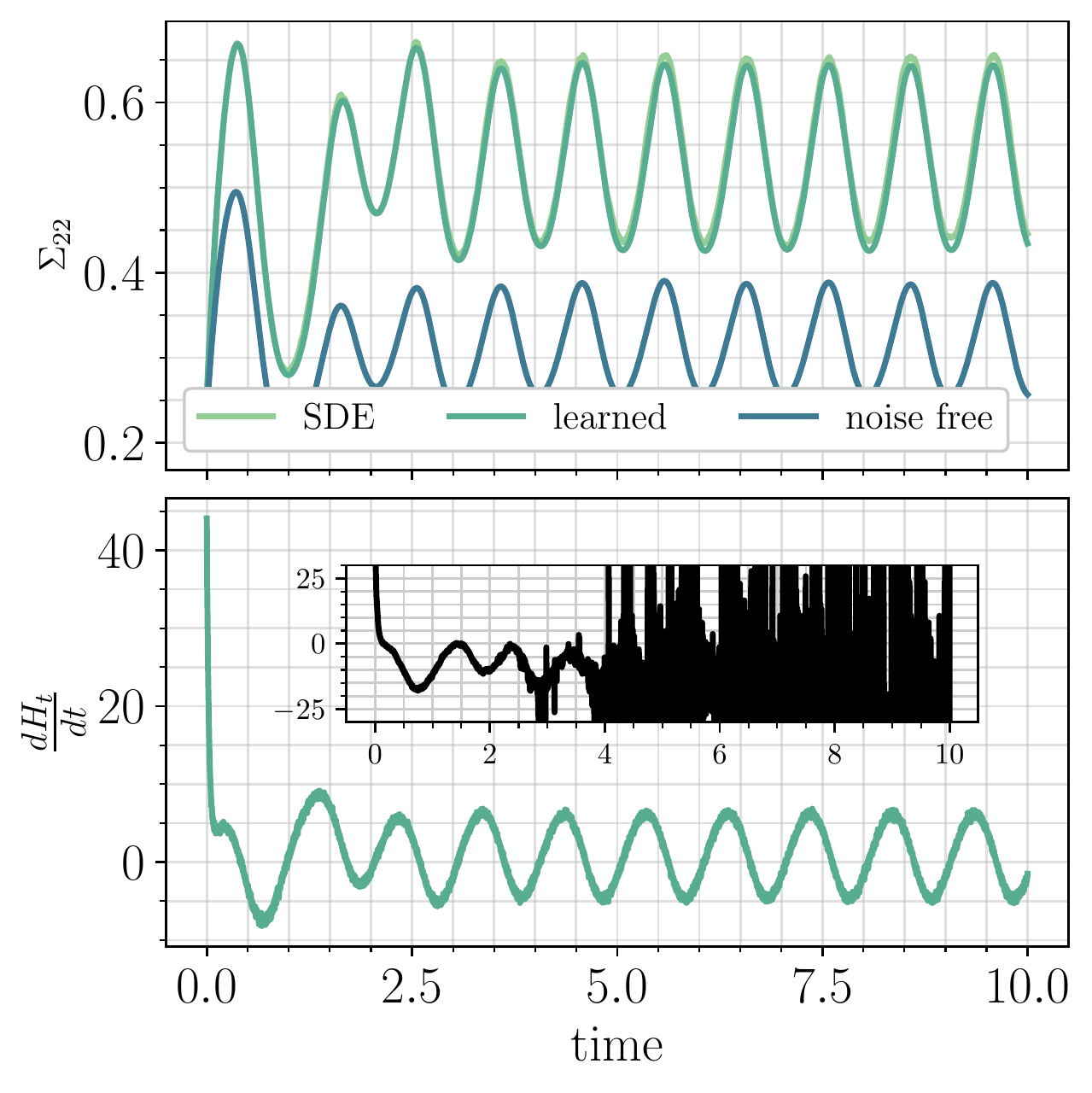}
    \put(3, 98){\textbf{C}}
    \put(3, 50){\textbf{E}}
    \end{overpic}
\end{tabular}
    \caption{\textit{A system of $N=5$ soft-spheres in an anharmonic trap:} (A) Example particle trajectories in the case of a rotating trap. Trap position shown with a red star. Movies of the circular and linear motion can be viewed \href{https://drive.google.com/drive/folders/1K0PFPecUvUez8gIisNaNyeOaFvMBZ_TE?usp=share_link}{here} and \href{https://drive.google.com/drive/folders/1oWs_QYlgNpi58cGvwyQa2i1a8tMamset?usp=share_link}{here}, respectively.
    (B/C) A single component of the covariance of the samples, in the case of a rotating trap in B and a linearly oscillating trap in C. The learned system agrees well with the SDE, while the noise-free system under-predicts the moments.
    (D/E) Prediction of the entropy production rate for a rotating trap in D and linearly oscillating trap in E. Main figure depicts the prediction obtained from SBTM, while the inset depicts the prediction obtained when learning on samples from the SDE. SBTM captures the temporal evolution of the entropy production rate, while learning on the SDE is initially offset and later divergent. 
    }
    \label{fig:anharmonic_particles}
\end{figure}

\paragraph{Setup.} Here, we consider a system of $N=5$ physical particles in an \textit{anharmonic} trap in dimension $\bar{d}=2$ that exhibit soft-sphere repulsion. This system gives rise to a $10$-dimensional~\eqref{eq:fpe}, which is significantly too high for standard PDE solvers. The 
stochastic dynamics is given by
\begin{equation}
\label{eqn:anharmonic_SDE}
\begin{aligned}
    dX^{(i)}_t &= 4B \big(\beta_t - X_t^{(i)} \big)|X_t^{(i)} - \beta_t|^2dt \nonumber\\
    &+ \frac{A}{N r^2}\sum_{j=1}^N \big(X^{(i)}_t - X^{(j)}_t\big)\exp\left(-\frac{|X^{(i)}_t - X^{(j)}_t|^2}{2r^2}\right)dt + \sqrt{2D}\,dW_t,\:\:\: i=1, \hdots, N, 
\end{aligned}
\end{equation}
where $\beta_t$ again represents a moving trap, $A>0$ sets the strength of the repulsion between the spheres, $r$ sets their size, and $B>0$ sets the strength of the trap. We set $\beta(t) = a(\cos\pi \omega t, \sin \pi \omega t)^\T$ or $\beta(t) = a(\cos\pi \omega t, 0)^\T$ with $a = 2, \omega = 1, D = 0.25, A=10$, and $r = 0.5$. We fix $B = D/R^2$ with $R = \sqrt{\gamma N}r$ and $\gamma = 5.0$. This ensures that the trap scales with the number of particles and that they have sufficient room in the trap to generate a complex dynamics. The circular case converges to a distribution $\rho_t^* = \rho^* \circ Q_t$ that can be described as a fixed distribution $\rho^*$ composed with a time-dependent rotation $Q_t$, and hence the entropy production rate converges to zero by change of variables. The linear case does not exhibit this kind of convergence, and the entropy production rate should oscillate around zero as the particles are repeatedly pushed and pulled by the trap. We make use of the same network architecture as in Sec.~\ref{sec:harmonic}.

\paragraph{Results.}
Similar to Section~\ref{sec:harmonic}, an example trajectory from the learned system, the SDE \eqref{eqn:anharmonic_SDE}, and the noise-free system obtained by setting $D = 0$ are shown in Figure~\ref{fig:anharmonic_particles}A in the circular case. The learned particle trajectories exhibit an intuitive circular motion when compared to the SDE trajectory. When compared to the noise-free system, the learned trajectories exhibit a greater amount of spread, which enables the deterministic dynamics to accurately capture the statistics of the stochastic dynamics. Numerical estimates of a single component of the covariance and of the entropy production rate are shown in Figure~\ref{fig:anharmonic_particles}B/C, with all moments shown in Appendix~\ref{app:exp:anharmonic}. The learned and SDE systems accurately capture the covariance, while the noise-free system underestimates the covariance in both the linear and the circular case. The prediction of the entropy production rate via Algorithm~\ref{alg:sbtm} is reasonable in both cases, exhibiting the expected convergence to and oscillation around zero in the circular and linear cases, respectively. In the inset, we show the prediction of the entropy production rate when learning on samples from the SDE; the prediction is initially offset, and later becomes divergent. We found that this behavior was generic when training on the SDE, but never observed it when training on self-consistent samples.

\subsection{An active swimmer}
\begin{figure}[!t]
    \centering
    \includegraphics[width=.5\textwidth]{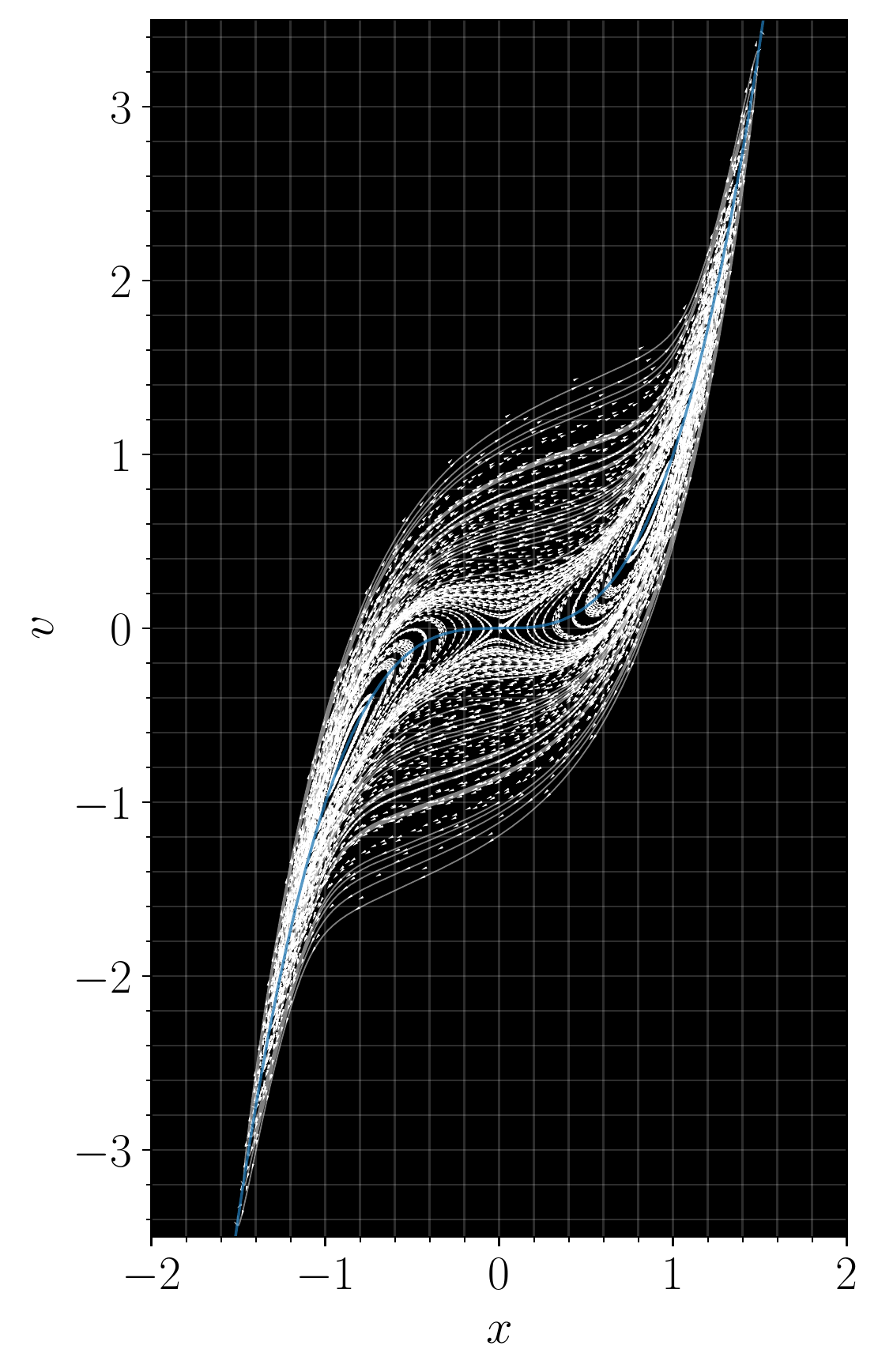}
    \caption{\textit{An active swimmer: probability flow phase portrait.} Phase portrait of the probability flow, computed with parameters frozen at the fixed time $t = 10/\gamma$. Low-opacity curves depict closed limit cycles, while arrows indicate the direction of the probability flow. The phase portrait reveals non-equilibrium steady-state currents, both within and between the two modes. The nullcline $v = x^3$ passes through the two modes (shown in blue), with an unstable equilibrium at the origin.}
    \label{fig:swimmer_pp}
\end{figure}

\begin{figure}[!t]
    \centering
    \includegraphics[width=\textwidth]{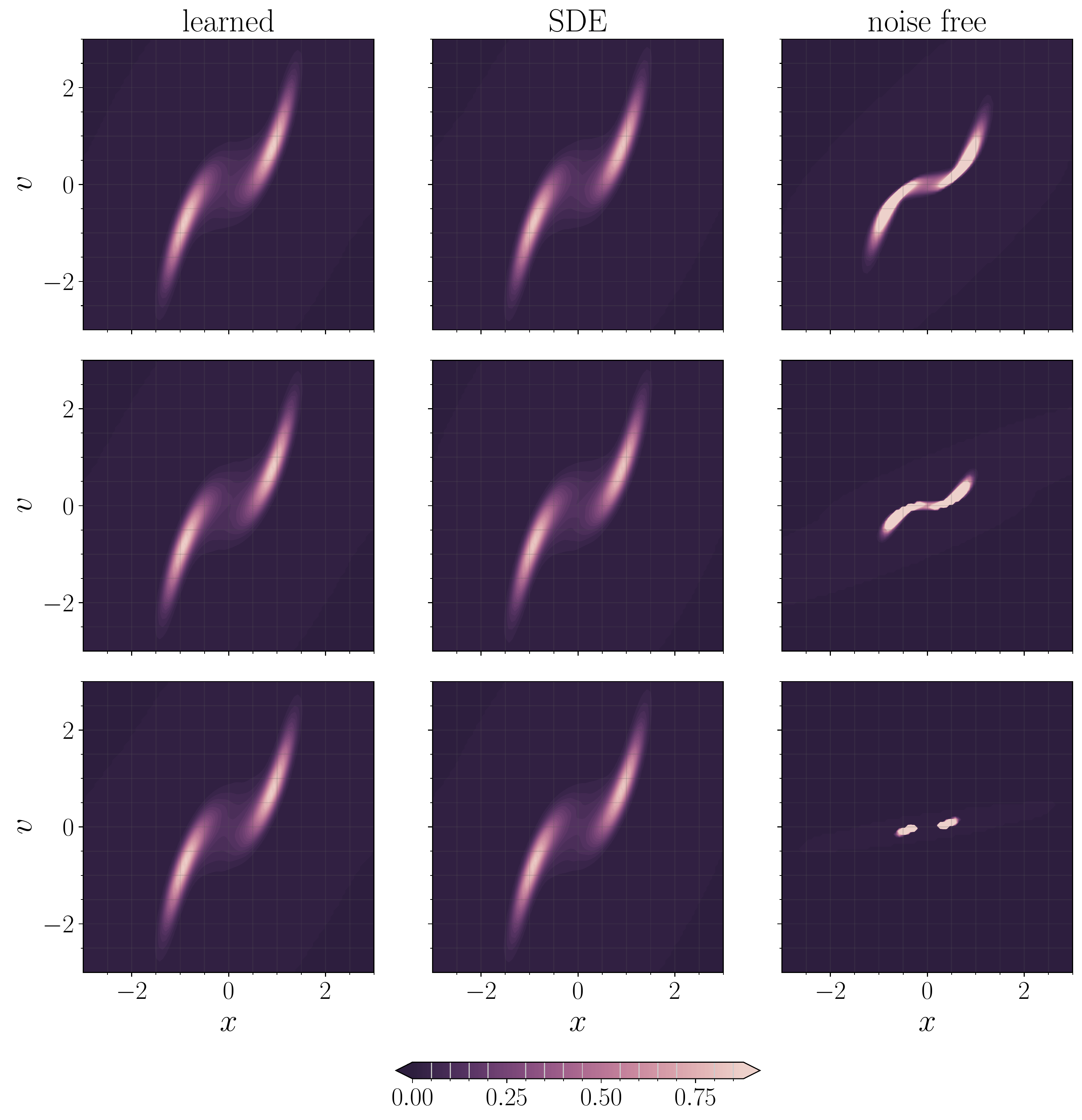}
    \caption{\textit{An active swimmer: kernel density estimates.} PDFs computed via kernel density estimation in the $xv$ plane. Columns denote solution type and rows denote snapshots in time ($t = 0.5/\gamma, 1.5/\gamma$, and $3.0/\gamma$, respectively). The KDE reveals bimodality in the probability density brought about by the activity of the particle. The noise free system becomes too concentrated around the nullcline $v = x^3$, and does not accurately capture the shape of the SDE and learned solutions, while the SDE and learned solutions are nearly identical.}
    \label{fig:swimmer_kde}
\end{figure}

\paragraph{Setup} We now consider a model from the physics of active matter, which describes the motion of a single motile swimmer in an anharmonic trap. The swimmer can be thought of as a run-and-tumble bacterium~\cite{tailleur_statistical_2008}; it travels in a fixed direction for a fluctuating duration before picking a new direction at random in which to swim. The system is two-dimensional, and is given by the stochastic differential equation for the position $x$ and velocity $v$
\begin{equation}
    \label{eqn:swimmer}
    \begin{aligned}
        dx &= \left(-x^3 + v\right)dt,\\
        dv &= -\gamma v dt + \sqrt{2\gamma D}dW_t.
    \end{aligned}
\end{equation}
While low-dimensional,~\eqref{eqn:swimmer} exhibits convergence to a non-equilibrium statistical steady state in which the probability current $j_t(x) = v_t(x)\rho_t(x)$ is non-zero. Here, we show that sequential SBTM is capable of accurately capturing such currents, which is necessary to resolve the dynamics of the Fokker-Planck equation: if our goal were solely to sample at equilibrium, it would be sufficient to freeze the samples after an initial transient. Moreover, we show that the method preserves the stationary distribution over long times relative to the persistence time $1/\gamma$ of the swimmer, and does not display appreciable accumulation of error.

We set $\gamma = 0.1$ and $D = 1.0$. Because noise only enters the system through the velocity variable $v$ in \eqref{eqn:swimmer}, the score can be taken to be one-dimensional, which is equivalent to learning the score only in the range of the rank-deficient diffusion matrix.  Further details on the architecture can be found in Appendix~\ref{app:swimmer}.

\paragraph{Results}
A phase portrait for the learned probability flow dynamics is shown in Figure~\ref{fig:swimmer_pp}, computed by rolling out an additional set of $50$ trajectories for time $5/\gamma$ with a fixed set of parameters (after learning for time $10/\gamma$). The phase portrait depicts closed limit cycles between and centered within the modes reminiscent of the classical phase portrait for the pendulum. Here, the closed limit cycles correspond to non-equilibrium currents that preserve the steady-state density.

A kernel density estimate for the distribution of samples produced by the learned system, the stochastic system, and the noise-free systems are shown in Figure~\ref{fig:swimmer_kde}, which demonstrate that the distribution of the learned samples qualitatively matches the distribution of the SDE samples. Comparatively, the noise-free system grows overly concentrated with time, ultimately converging to a singular dirac measure at the origin. A movie of the motion of the samples $(x_i(t), v_i(t))_{t\geq 0}$ over a duration $10/\gamma$ in phase space can be seen \href{https://drive.google.com/file/d/1ysbVpmMN1SHrEzOJmkoPzwFaKqeXEjtw/view?usp=share_link}{at this link}. The movie highlights convergence of the learned solution to one with a non-zero steady-state probability current that qualitatively matches that of the SDE, but which enjoys more interpretable sample trajectories.

\section{Outlook and conclusions}
Building on the toolbox of score-based diffusion recently developed for generative modeling, we introduced a related approach -- score-based transport modeling (SBTM) -- that gives an alternative to simulating the corresponding SDE to solve the Fokker-Planck equation. While SBTM is more costly than integration of the SDE because it involves a learning component, it gives access to quantities that are not directly accessible from the samples given by integrating the SDE, such as pointwise evaluation of the PDF, the probability current, or the entropy. Our numerical examples indicate that SBTM is scalable to systems in high dimension where standard numerical techniques for partial differential equations are inapplicable. The method can be viewed as a deterministic Lagrangian integration method for the Fokker-Planck equation, and our results show that its trajectories are more easily interpretable than the corresponding trajectories of the SDE. 

\bibliography{refs}
\bibliographystyle{plainnat}
\clearpage

\appendix
\counterwithin{figure}{section}
\counterwithin{equation}{section}
\counterwithin{theorem}{section}
\counterwithin{proposition}{section}
\counterwithin{lemma}{section}
\counterwithin{remark}{section}

\section{Some basic formulas}
\label{app:basic}
Here, we derive some results linking the solution of the transport equation~\eqref{eq:transport}  with that of the probability flow equation~\eqref{eq:probflow}.

\subsection{Probability density and probability current}
We begin with a lemma.
\begin{lemma}
\label{th:pdf}
Let $\rho_t: \Omega \rightarrow \R_{\geq 0}$ satisfy the transport equation
\begin{equation}
    \label{eq:transp:fpe:a}
    \partial_t \rho_t(x) = - \nabla \cdot \left(v_t(x) \rho_t(x)\right).
\end{equation}
Assume that $v_t(x)$ is $C^2$ in both $t$ and $x$  for $t\ge0$ and globally Lipschitz in $x$. Then, given any $t,t' \ge 0$, the solution of \eqref{eq:transp:fpe:a} satisfies
\begin{equation}
    \label{eq:pdf:local:a}
    \rho_t(x) =  \rho_{t'}(X_{t,t'}(x)) \exp\left( -\int_{t'}^t \nabla \cdot v_\tau(X_{t,\tau}(x)) d\tau\right)
\end{equation}
where $X_{\tau,t}$ is the probability flow solution to~\eqref{eq:probflow}.
In addition, given any test function $\phi: \Omega \to \R$, we have
\begin{equation}
    \label{eq:expedct:a}
    \int_\Omega \phi(x) \rho_t(x) dx = \int_\Omega \phi(X_{t',t}(x)) \rho_{t'}(x) dx.
\end{equation}
\end{lemma}
In words, Lemma~\ref{th:pdf} states that an evaluation of the PDF $\rho_t$ at a given point $x$ may be obtained by evolving the probability flow equation \eqref{eq:probflow} backwards to some earlier time $t'$ to find the point $x'$ that evolves to $x$ at time $t$, assuming that $\rho_{t'}(x')$ is available. In particular, for $t'=0$, we obtain
\begin{equation}
    \label{eq:pdf:local:aa}
    \rho_t(x) =  \rho_0(X_{t,0}(x)) \exp\left( -\int_{0}^t \nabla \cdot v_\tau(X_{t,\tau}(x)) d\tau\right),
\end{equation}
and
\begin{equation}
    \label{eq:expedct:aa}
    \int_\Omega \phi(x) \rho_t(x) dx = \int_\Omega \phi(X_{0,t}(x)) \rho_{0}(x) dx.
\end{equation}
Since the probability current is by definition $v_t(x)\rho_t(x)$, using~\eqref{eq:pdf:local:aa} to express~$\rho_t(x)$ also gives the follwing equation for the current:
\begin{equation}
    \label{eq:current:local}
    v_t(x) \rho_t(x) = v_t(x) \rho_0(X_{t,0}(x)) \exp\left( -\int_0^t \nabla \cdot v_\tau(X_{\tau,t}(x)) d\tau\right).
\end{equation}

\begin{proof} The assumed $C^2$ and globally Lipschitz conditions on $v_t$ guarantee global existence (on $t\ge0$) and uniqueness of the solution to~\eqref{eq:probflow}. Differentiating $\rho_t(X_{t',t}(x))$ with respect to $t$ and using \eqref{eq:probflow} and \eqref{eq:transp:fpe:a} we deduce
\begin{equation}
    \label{eq:diff:a}
    \begin{aligned}
    \frac{d}{dt} \rho_t(X_{t',t}(x)) &= \partial_t \rho_t(X_{t',t}(x)) + \frac{d}{dt}X_{t',t}(x) \cdot \nabla \rho_t(X_{t',t}(x)) \\
    &= \partial_t \rho_t(X_{t',t}(x)) + v_t(X_{t',t}(x)) \cdot \nabla \rho_t(X_{t',t}(x))\\
    & = - \nabla \cdot v_t(X_{t',t}(x)) \, \rho_t(X_{t',t}(x))
    \end{aligned}
\end{equation}
Integrating this equation in $t$ from $t=t'$ to $t=t$ gives
\begin{equation}
    \label{eq:diff:b}
    \begin{aligned}
    \rho_t(X_{t',t}(x)) &= \rho_{t'}(x) \exp\left( - \int_{t'}^t \nabla \cdot v_\tau(X_{t',\tau}(x)) d\tau\right)
    \end{aligned}
\end{equation}
Evaluating this expression at $x= X_{t,t'}(x)$ and using the group properties (i) $X_{t',t}(X_{t,t'}(x)) = x$ and (ii) $X_{t',\tau}(X_{t,t'}(x)) = X_{t,\tau}(x)$ gives~\eqref{eq:pdf:local:a}. Equation~\eqref{eq:expedct:a} can be derived by using~\eqref{eq:pdf:local:a} to express $\rho_t(x)$ in the integral at the left hand-side, changing integration variable $x\to X_{t',t}(x)$ and noting that the factor $ \exp\left( -\int_{t'}^t \nabla \cdot v_\tau(X_{t,\tau}(x))\right)$ is precisely the Jacobian of this change of variable. The result is the integral at the right hand-side of~\eqref{eq:expedct:a}.
\end{proof}

Lemma~\ref{th:pdf} also holds locally in time for any $v_t(x)$ that is $C^2$ in both $t$ and $x$. In particular, it holds locally if we set $s_t(x) = \nabla \log \rho_t(x)$ and if we assume that $\rho_0(x)$ is (i) positive everywhere on $\Omega$ and (ii) $C^3$ in $x$. In this case, \eqref{eq:transp:fpe:a} is the Fokker-Planck equation~\eqref{eq:fpe} and \eqref{eq:pdf:local:a} holds for the solution to that equation.

\subsection{Calculation of the differential entropy}

We now consider computation of the differential entropy, and state a similar result.
\begin{lemma}
\label{th:entropy}
Assume that  $\rho_0: \Omega \rightarrow \R_{\geq 0}$ is positive everywhere on $\Omega$  and $C^3$ in its argument. Let $\rho_t: \Omega \rightarrow \R_{\geq 0}$ denote the solution to the Fokker Planck equation~\eqref{eq:fpe} (or equivalently, to the transport equation~\eqref{eq:transp:fpe:a} with $s_t(x) = \nabla \log \rho_t(x)$ in the definition of $v_t(x)$). Then the differential entropy $H_t = -\int_\Omega \log \rho_t(x) \, \rho_t(x) dx$ can expressed as
\begin{equation}
    \label{eq:entropy:1:a}
    H_t = -\int_\Omega \log \rho_t(X_{0,t}(x))\, \rho_0(x) dx = H_0 + \int_0^t \int_\Omega \nabla \cdot v_\tau(X_{0,\tau}(x)) \rho_0(x) dx d\tau
\end{equation}
or
\begin{equation}
    \label{eq:entropy:3:a}
    H_t = H_0 - \int_0^t \int_\Omega s_\tau(X_{0,\tau}(x)) \cdot v_\tau(X_{0,\tau}(x)) \rho_0(x) dx d\tau
\end{equation}
\end{lemma}
\begin{proof}
We first derive~\eqref{eq:entropy:1:a}. Observe that applying~\eqref{eq:expedct:aa} with $\phi = \log \rho_t$ leads to the first equality. The second can then be deduced from ~\eqref{eq:pdf:local:aa}. To derive~\eqref{eq:entropy:3:a}, notice that from~\eqref{eq:transp:fpe:a},
\begin{equation}
    \label{eq:ent:a}
    \begin{aligned}
    \frac{d}{dt} H_t & =  \int_\Omega \log \rho_t(x) \nabla \cdot \left(v_t(x) \rho_t(x)\right) dx,\\
    & = -\int_\Omega \nabla \log \rho_t(x)  \cdot v_t(x) \rho_t(x) dx,\\
    & = -\int_\Omega s_t(x)  \cdot v_t(x) \rho_t(x) dx
    \end{aligned}
\end{equation}
Above, we used integration by parts to obtain the second equality and $s_t = \nabla \log \rho_t$ to get the third. Now, using~\eqref{eq:expedct:aa} with $\phi = s_t  \cdot v_t$ integrating the result gives~\eqref{eq:entropy:3:a}.
\end{proof}

\subsection{Resampling of $\rho_t$ at any time $t$} 

If the score $s_t \approx \nabla \log \rho_t$ is known to sufficient accuracy, $\rho_t$ can be resampled at any time $t$ using the dynamics
\begin{equation}
    \label{eq:sde:artif}
    dX_\tau = s_t(X_\tau) d\tau + d W_\tau.
\end{equation}
In~\eqref{eq:sde:artif}, $\tau$ is an artificial time used for sampling that is distinct from the physical time in~\eqref{eq:sde}. For $s_t = \nabla \log \rho_t$, the equilibrium distribution of~\eqref{eq:sde:artif} is exactly $\rho_t$. In practice, $s_t$ will be imperfect and will have an error that increases away from the samples used to learn it; as a result,~\eqref{eq:sde:artif} should be used near samples for a fixed amount of time to avoid the introduction of additional errors.

\section{Further details on Score-Based Transport Modeling}
\label{app:proofs}
\subsection{Bounding the KL divergence}
\label{app:ent:prod}
Let us restate Proposition~\ref{prop:entropy} for convenience:
\sbtment*
\begin{proof}
 By assumption, $\rho_t$ solves~\eqref{eq:transport} and $\rho_t^*$ solves~\eqref{eq:fpe}. Denote by $v_t(x) = b_t(x) - D_t(x) s_t(x)$ and $v^*_t(x) = b_t(x) - D_t(x) s^*_t(x)$ with $s^*_t(x)=\nabla \log \rho_t^*(x)$. Then, we have
 \begin{equation*}
     \label{eq:ent:s1}
     \begin{aligned}
     \frac{d}{dt} \kl{\rho_t}{\rho^*_t} &= \frac{d}{dt} \int_\Omega \log\left(\frac{\rho_t(x)}{\rho_t^*(x)}\right) \rho_t(x) dx,\\
     & = -\int_\Omega \frac{\rho_t(x)}{\rho_t^*(x)} \partial_t \rho^*_t(x) dx + \int_\Omega \log \left(\frac{\rho_t(x)}{\rho_t^*(x)}\right) \partial_t \rho_t(x)dx,\\
     & = -\int_\Omega v_t^*(x) \cdot \nabla \left(\frac{\rho_t(x)}{\rho_t^*(x)}\right)  \rho^*_t(x) dx + \int_\Omega v_t(x) \cdot \nabla \log \left(\frac{\rho_t(x)}{\rho_t^*(x)}\right) \rho_t(x) dx,\\
     & = -\int_\Omega \left(v_t^*(x) - v_t(x)\right) \cdot \left(\nabla \log \rho_t(x)- \nabla \log \rho^*_t(x)\right) \rho_t(x) dx,\\
     & = \int_\Omega \left(s_t^*(x) - s_t(x)\right) \cdot D_t(x) \left(\nabla \log \rho_t(x)- s^*_t(x)\right) \rho_t(x) dx.
     \end{aligned}
 \end{equation*}
Above, we used integration by parts to obtain the third equality. Now, dropping function arguments for simplicity of notation, we have that
\begin{equation*}
    \label{eq:ent:s3}
    \begin{aligned}
    |\nabla\log\rho_t - s_t|^2_{D_t} & = |\nabla\log\rho_t -s^*_t + s^*_t - s_t|^2_{D_t},\\
    & = |\nabla\log\rho_t - s^*_t|^2_{D_t} + |s^*_t - s_t|^2_{D_t} + 2 (\nabla\log\rho_t -s^*_t)\cdot D_t(s^*_t - s_t),\\
    & \ge  2 (\nabla\log\rho_t - s^*_t)\cdot D_t (s^*_t - s_t).
    \end{aligned}
\end{equation*}
Hence, we deduce that
\begin{equation}
     \label{eq:ent:s4}
     \begin{aligned}
     &\frac{d}{dt} \kl{\rho_t}{\rho^*_t} \le \frac12  \int_\Omega |s_t(x) - \nabla\log\rho_t(x)|^2_{D_t(x)} \rho_0(x) dx.
     \end{aligned}
\end{equation}
\end{proof}

\subsection{SBTM in the Eulerian frame}
\label{app:sbtm:euler}
The Eulerian equivalent of Proposition~\ref{prop:loss2} can be stated as:

\begin{proposition}[SBTM in the Eulerian frame]
\label{prop:loss1} 
Assume that the conditions listed in Sec.~\ref{sec:assumptions} hold. Fix $T\in(0,\infty]$ and consider the optimization problem
\begin{equation}
\label{eq:sbtm2}
    \tag{SBTM2}
    \begin{aligned}
    &\min_{\{s_t:t\in [0,T)\}} \int_0^T \int_\Omega \left|s_t(x) -  \nabla \log \rho_t(x)\right|_{D_t(x)}^2 \rho_t(x) dx dt\\[6pt]
    & \text{subject to:} \quad \partial_t \rho_t(x) = - \nabla \cdot \left( v_t(x) \rho_t (x) \right), \:\:\: x \in \Omega
    \end{aligned}
\end{equation}
with $v_t(x) = b_t(x) - D_t(x) s_t(x)$. Then every minimizer of \eqref{eq:sbtm2} satisfies $D_t(x)s_t^*(x) = D_t(x)\nabla \log \rho_t^*(x)$ where $\rho^*_t: \Omega \to\R_{> 0}$ solves~\eqref{eq:fpe}.
\end{proposition}
In words, this proposition states that solving the constrained optimization problem~\eqref{eq:sbtm2} is equivalent to solving the Fokker-Planck equation~\eqref{eq:fpe}.
\begin{proof}
The constrained minimization problem~\eqref{eq:sbtm2} can be handled by considering the extended objective
\begin{equation}
\label{eq:prob:2:aa}
    \begin{aligned}
    \int_0^T \int_\Omega \left(\left|s_t(x) -  \nabla \log \rho_t(x)\right|_{D_t(x)}^2 \rho_t(x) + \mu_t(x) \left(\partial_t \rho_t(x) + \nabla \cdot \left(v_t(x) \rho_t(x)\right)\right) \right) dx dt
    \end{aligned}
\end{equation}
where $v_t(x) = b_t(x) - D_t(x) s_t(x)$ and $\mu_t:\R^d\rightarrow\R_{\geq 0}$ is a Lagrange multiplier. The Euler-Lagrange equations associated with~\eqref{eq:prob:2:aa} read
\begin{equation}
\label{eq:prob:2:aaa}
    \begin{aligned}
     \partial_t \rho_t(x) &= - \nabla \cdot \left(v_t(x)\rho_t(x)\right)\\
     \partial_t \mu_t(x)  &= v_t(x)^\T \nabla \mu_t(x) + |s_t(x)|_{D_t(x)}^2 - |\nabla \log \rho_t|_{D_t(x)}^2 \\
     &\qquad\qquad + 2 \nabla\cdot\left[D_t(x)\left(s_t(x) - \nabla \log \rho_t(x)\right)\right] , \\
     0 &= \mu_T(x),\\
     0 &= D_t(x)\left(s_t(x) - \nabla\log\rho_t(x)\right)\rho_t(x) + \tfrac12 D_t(x) \nabla\mu_t(x)\rho_t(x)
    \end{aligned}
\end{equation}
Clearly, these equations will be satisfied if $s^*_t(x) =  \nabla \log\rho^*_t(x)$ for all $x \in \Omega$, $\mu^*_t(x) = 0$ for all $x$, and $\rho_t^*$ solves~\eqref{eq:fpe}. This solution is also a global minimizer, because it zeroes the value of the objective. Moreover, all global minimizers must satisfy $D_t(x)s^*_t(x) = D_t(x)\nabla \log\rho^*_t(x)$ ($\rho_t-$almost everywhere), as this is the \textit{only} way to zero the objective.

It is also easy to see that there are no other local minimizers. To check this, we can use the fourth equation to write
\begin{equation*}
    D_t(x)(s_t(x) - \nabla\log\rho_t(x)) = \tfrac{1}{2}D_t(x)\nabla\mu_t(x).
\end{equation*}
Then,
\begin{align*}
    |s_t(x)|_{D_t(x)}^2 - |\nabla\log\rho_t(x)|_{D_t(x)}^2 = \tfrac{1}{2}\left(s_t(x) + \nabla\log\rho_t(x)\right)^\T D_t(x) \nabla\mu_t(x).
\end{align*}
This reduces the first three equations to
\begin{equation}
\label{eq:prob:2:aaaa}
    \begin{aligned}
     \partial_t \rho_t(x) &= - \nabla \cdot \left(b_t(x)\rho_t(x) - D_t(x) \nabla \rho_t(x) - \tfrac{1}{2}\rho_t D_t(x) \nabla \mu_t(x)\right) \\
     \partial_t \mu_t &= \left(b_t(x) - D_t(x)\nabla\log\rho_t(x) - \tfrac{1}{2}D_t(x)\nabla\mu_t(x)\right)^\T\nabla\mu_t(x)\\
     &\qquad + \nabla\cdot\left(D_t(x)\nabla\mu_t(x)\right) + \tfrac{1}{2}\left(s_t(x) + \nabla\log\rho_t(x)\right)^\T D_t(x)\nabla\mu_t(x).\\
     \mu_T(x) &= 0.
    \end{aligned}
\end{equation}
Since the equation for $\mu_t$ is homogeneous in $\mu_t$ and $\mu_T=0$, we must have $\mu_t=0$ for all $t\in[0,T)$, and the equation for $\rho_t$ reduces to~\eqref{eq:fpe}.
\end{proof}

\subsection{SBTM in the Lagrangian frame}
\label{app:sbtm:lag}

As stated, Proposition~\ref{prop:loss1} is not practical, because it is phrased in terms of the density $\rho_t$. The following result demonstrates that the transport map identity~\eqref{eq:rhot:rho0} can be used to re-express Proposition~\ref{prop:loss1} entirely in terms of known quantities.

\sbtmtwo*
\begin{proof}
Let us first show that $G_t(x)=\nabla \log \rho_t(X_t(x))$ satisfies~\eqref{eq:XtGt} if
$\rho_t= X_t\sharp \rho_0,$ i.e. if $\rho_t$ satisfies the transport equation~\eqref{eq:transport}. Since~\eqref{eq:transport} implies that
\begin{equation}
    \label{eq:g:1a}
    \partial_t \log \rho_t(x) + v_t(x) \cdot \nabla \log \rho_t(x) = - \nabla \cdot v_t(x),
\end{equation}
taking the gradient gives
\begin{equation}
    \label{eq:g:2a}
    \partial_t \nabla \log \rho_t(x) + [\nabla v_t(x)]^\T \nabla \log \rho_t(x) + \nabla \nabla \log \rho_t(x) \cdot v_t(x) = - \nabla \nabla \cdot v_t(x) .
\end{equation}
Therefore $G_t(x) = \nabla \log \rho_t(X_t(x))$ solves
\begin{equation}
    \label{eq:g:3a}
    \begin{aligned}
    \frac{d}{dt} G_t(x) & = \partial_t \nabla \log \rho_t(X_t(x)) + \nabla \nabla \log \rho_t(X_t(x))\cdot \frac{d}{dt} X_t(x), \\
    & = \partial_t \nabla \log \rho_t(X_t(x)) + \nabla \nabla \log \rho_t(X_t(x))\cdot v_t(x),\\
    & = - \nabla \nabla \cdot v_t(X_{t}(x)) - [\nabla  v_t(X_{t}(x))]^\T \nabla \log \rho_t(X_t(x)) ,
    \end{aligned}
\end{equation}
which recovers the equation for $G_t(x)$ in~\eqref{eq:XtGt}. Hence, the objective in~\eqref{eq:sbtm} can also be written as
\begin{equation}
    \label{eq:loss:prob:flow:a2}
    \begin{aligned}
    &\int_0^T \int_\Omega \left|s_t(X_t(x)) - \nabla \log \rho_t(X_t(x))\right|^2 \rho_0(x) dx dt\\
    &\qquad\qquad = \int_0^T \int_\Omega \left|s_t(x) - \nabla \log \rho_t(x)\right|^2 \rho_t(x) dx dt
    \end{aligned}
\end{equation}
where the second equality follows from \eqref{eq:expedct:aa} if $\rho_t(x)$ satisfies~\eqref{eq:transp:fpe:a}. Hence, \eqref{eq:sbtm} is equivalent to \eqref{eq:sbtm2}. The bound on $\kl{X_T\sharp\rho_0}{\rho_T^*}$ follows from~\eqref{eq:ent:prod:n}.\hfill
\end{proof}
\paragraph{Adjoint equations.} In terms of a practical implementation, the objective in~\eqref{eq:sbtm2} can be evaluated by generating samples $\{x_i\}_{i=1}^n$ from $\rho_0$ and solving the equations for $X_t$ and $G_t$ using the initial conditions $X_0(x_i) = x_i$ and $G_0(x_i) = \nabla \log \rho_0(x_i)$. Note that evaluating this second initial condition only requires one to know $\rho_0$ up to a normalization factor. To evaluate the gradient of the objective, we can introduce equations adjoint to those for $X_t$ and $G_t$. They read, respectively
\begin{equation}
    \label{eq:adjoint:a}
    \begin{aligned}
    \frac{d}{dt}\theta_{t}(x) + [\nabla v_t(X_{t}(x))]^\T \theta_t(x) &= \eta_t(x) \cdot \nabla  \nabla v_t(X_{t}(x)) G_t(x) \\
    &\qquad + \eta_t(x) \cdot \nabla \nabla \nabla v_t(X_{t}(x))G_t(x)\\
    &\qquad\qquad + 2\nabla s_t(X_t(x))(s_t(X_t(x))-G_t(x)), \\
    \theta_{T}(x) &= 0\\
    \frac{d}{dt}\eta_{t}(x) - \nabla  v_t(X_{t}(x)) \eta_t(x) &= 2(G_t(x)-s_t(X_t(x))),\\
    \eta_T(x) &= 0.
\end{aligned}
\end{equation}
In terms of these functions, the gradient of the objective is the gradient with respect to $s_t(x)$ (or the parameters in this  function when it is modeled by a neural network) of the extended objective:
\begin{equation}
    \label{eq:extended:obj:lag:a}
    \begin{aligned}
    L[s_t] &= \int_0^T \int_\Omega \left|s_t(X_t(x)) - G_t(x)\right|^2 \rho_0(x) dx dt \\
    &\qquad + \int_0^T \int_\Omega \theta_t(x) \cdot\left (\dot X_{t}(x) - v_t(X_{t}(x)) \right) \rho_0(x) dx dt\\
    &\qquad + \int_0^T \int_\Omega \eta_t(x) \cdot\Big(\dot G_t(x) + [\nabla  v_t(X_{t}(x))]^\T G_t(x) \\
    &\qquad\qquad\qquad\qquad + \nabla \nabla \cdot v_t(X_{t}(x))\Big) \rho_0(x) dx dt,
    \end{aligned}
\end{equation}
where $v_t(x) = b_t(x) - D_t(x) s_t(x)$.

\subsection{Sequential SBTM}
\label{app:sbtm:seq} 

Let us restate Proposition~\ref{prop:loss3} for convenience:
\sbtmthree*
\begin{proof}
If $X_t \sharp \rho_0=\rho_t$, then by definition we have the identity
\begin{align}
    \label{eq:idenity:loss}
    &\int_\Omega \left(|s_t(X_t(x))|_{D_t(X_t(x))}^2 +2 \nabla \cdot \left(D_t(X_t(x))s_t(X_t(x))\right) \right) \rho_0(x) dx \nonumber\\
    &\qquad\qquad\qquad = \int_\Omega \left(|s_t(x)|_{D_t(x)}^2 +2 \nabla \cdot \left(D_t(x) s_t(x) \right)\right) \rho_t(x) dx.
\end{align}
This means that the optimization problem in \eqref{eq:sbtm3} is equivalent to 
\begin{equation*}
     \min_{s_t} \int_\Omega \left(|s_t(x)|_{D_t(x)}^2 +2 \nabla \cdot \left(D_t(x) s_t(x) \right)\right) \rho_t(x) dx.
\end{equation*}
All minimizers $s_t^*$ of this optimization problem satisfy $D_t(x)s^*_t(x) = D_t(x)\nabla\log \rho_t(x)$. Hence, by~\eqref{eq:transport},
\begin{equation}
    \partial_t \rho_t(x) = -\nabla\cdot \left(b_t(x)\rho_t(x) - D_t(x)\nabla\rho_t(x)\right)
\end{equation}
which recovers~\eqref{eq:fpe}, so that $\rho_t(x) = \rho_t^*(x)$ solves~\eqref{eq:fpe}.
\end{proof}

\subsection{Learning from the SDE}
\label{app:sde_learn}
In this section, we show that learning from the SDE alone -- i.e., avoiding the use of self-consistent samples -- does not provide a guarantee on the accuracy of $\rho_t$. We have already seen in~\eqref{eq:ent:prod:n} that it is sufficient to control $\int_0^{T}\int_{\Omega} |s_t(x) - \nabla\log\rho_t(x)|_{D_t}^2\rho_t^*(x)dxdt$ to control $\kl{\rho_T}{\rho_T^*}$. The proof of Proposition~\ref{prop:entropy} shows that control on 
\begin{equation}
    \label{eq:SDE_score_matching}
    \int_0^T \int_{\Omega}|s_t(x) - \nabla\log\rho_t^*(x)|_{D_t(x)}^2\rho_t^*(x)dxdt,
\end{equation}
as would be provided by training on samples from the SDE, does not ensure control on $\kl{\rho_T}{\rho_T^*}$. The following proposition shows that control on~\eqref{eq:SDE_score_matching} does not guarantee control on $\kl{\rho_T^*}{\rho_T}$ either. An analogous result appeared in~\cite{lu2021maximum} in the context of SBDM for generative modeling; here, we provide a self-contained treatment to motivate the use of the sequential SBTM procedure discussed in the main text.
\begin{proposition}
    \label{prop:SDE_learn}
    Let $\rho_t:\Omega\rightarrow\R_{> 0}$ solve~\eqref{eq:transport}, and let $\rho_t^*:\Omega\rightarrow\R_{> 0}$ solve~\eqref{eq:fpe}. Then, the following equality holds
    \begin{equation}
    \label{eq:nobound}
    \begin{aligned}
        &\kl{\rho_T^*}{\rho_T} = \int_0^T\int_{\Omega}|s_t(x) - \nabla\log\rho_t^*(x)|_{D_t(x)}^2\rho_t^*(x)dxdt\\
        &\qquad + \int_0^T \int_{\Omega} \left(\nabla\log\rho_t(x) - s_t(x)\right)^\T D_t(x)\left(s_t(x) - \nabla\log\rho_t^*(x)\right)\rho_t^*(x)dxdt.
    \end{aligned}
    \end{equation}
\end{proposition}
Proposition~\ref{prop:SDE_learn} shows that minimizing the error between $s_t$ and $\nabla\log\rho_t^*$ on samples of $\rho_t^*$ leaves a remainder term, because in general $\nabla\log\rho_t \neq s_t$. The proof shows that we may obtain the simple upper bound
    \begin{equation}
    \begin{aligned}
        &\kl{\rho_T^*}{\rho_T} \leq \frac{1}{2}\int_0^T\int_{\Omega}|s_t(x) - \nabla\log\rho_t^*(x)|_{D_t(x)}^2\rho_t^*(x)dxdt\\
        &\qquad + \frac{1}{2}\int_0^T \int_{\Omega} |s_t(x) - \nabla\log\rho_t(x)|_{D_t(x)}^2\rho_t^*(x)dxdt.
    \end{aligned}
    \end{equation}
However, controlling the above quantity requires enforcing agreement between $s_t$ and $\nabla\log\rho_t$ in addition to $s_t$ and $\nabla\log\rho_t^*$; this is precisely the idea of SBTM.
\begin{proof}
By an analogous argument as in the proof of Proposition~\ref{prop:entropy}, we find
\begin{align*}
    \frac{d}{dt}\kl{\rho_t^*}{\rho_t} &= \int \left(\nabla\log\rho_t(x) - \nabla\log\rho_t^*(x)\right)^\T D_t(x)\left(s_t(x) - \nabla\log\rho_t^*(x)\right)\rho_t^*(x)dx
\end{align*}
Adding and subtracting $s_t(x)$ to the first term in the inner product and expanding gives
\begin{equation}
\begin{aligned}
     &\frac{d}{dt}\kl{\rho_t^*}{\rho_t} = \int_{\Omega}|s_t(x) - \nabla\log\rho_t^*(x)|^2\rho_t^*(x)dx\\
     &\qquad + \int_{\Omega} \left(\nabla\log\rho_t(x) - s_t(x)\right)^\T D_t(x)\left(s_t(x) - \nabla\log\rho_t^*(x)\right)\rho_t^*(x)dx,
 \end{aligned}
\end{equation}
Integrating from $0$ to $T$ completes the proof.
\end{proof}

\subsection{Denoising Loss}
\label{app:denoise}
The following standard trick can be used to avoid computing the divergence of $s_t(x)$:
\begin{lemma}
\label{th:denoise:a} Given $\xi=N(0,I)$, we have
\begin{equation}
    \label{eq:loss2:denoise:bb}
    \begin{aligned}
    &\lim_{\alpha\downarrow0}\alpha^{-1} \E \big(s_{t}(x+\alpha \xi) \cdot \xi \big) =  \nabla\cdot s_t(x),\\
    &\lim_{\alpha\downarrow0}\alpha^{-1} \E \big(s_{t}(x+ \alpha \sigma_t(x) \xi) \cdot \sigma_{t} (x) \xi \big) =  \tr\left(D_t(x) \nabla s_t(x)\right)
    \end{aligned}
\end{equation}
\end{lemma}

\begin{proof}
We have
\begin{equation}
    \label{eq:setp1:aa}
    \begin{aligned}
    \alpha^{-1} s_{t}(x+\alpha \xi) \cdot \xi &= \alpha^{-1} s_{t}(x) \cdot \xi + (\nabla s_t(x) \xi ) \cdot  \xi + o(\alpha)
    \end{aligned}
\end{equation}
The expectation of the first term on the right-hand side of this equation is zero; the expectation of the second gives the result in~\eqref{eq:loss2:denoise:bb}. Hence, taking the expectation of~\eqref{eq:setp1:aa} and evaluating the result in the limit as $\alpha\downarrow 0$ gives the first equation in~\eqref{eq:loss2:denoise:bb}. The second equation in~\eqref{eq:loss2:denoise:bb} can be proven similarly using $\sigma_t(x) \sigma_t(x)^\T = D_t(x)$.
\end{proof}
Replacing $\nabla\cdot s_t(x)$ in~\eqref{eq:sbtm3} with the first expression in~\eqref{eq:setp1:aa} for a fixed $\alpha > 0$ gives the loss
\begin{equation}
    \label{eqn:sliced}
    \mathcal{L}[s_t] = \E_\xi\left[\int_{\Omega}\left(|s_t(X_t(x))|^2 + \frac{2}{\alpha}s_t(X_t(x) + \alpha \xi)\cdot\xi\right)\rho_0(x)dx\right].
\end{equation}
Evaluating the square term at a perturbed data point recovers the denoising loss of~\citet{vincent_connection_2011}
\begin{equation}
    \label{eqn:denoising}
    \mathcal{L}[s_t] = \E_\xi\left[\int_{\Omega}\left|s_t(X_t(x) + \alpha\xi) + \frac{\xi}{\alpha}\right|^2\rho_0(x)dx\right].
\end{equation}
We can improve the accuracy of the approximation with a ``doubling trick'' that applies two draws of the noise of opposite sign to reduce the variance. This amounts to replacing the expectations in~\eqref{eq:loss2:denoise:bb} with
\begin{equation}
    \label{eq:loss2:denoise:2}
    \begin{aligned}
    &\tfrac12\alpha^{-1} \E \big[ s_{t}(x+\alpha \xi) \cdot \xi  - s_{t}(x- \alpha \xi) \cdot \xi \big],\\
    &\tfrac12\alpha^{-1} \E \big[s_{t}(x+ \alpha \sigma_t(x) \xi) \cdot \sigma_{t} (x) \xi - s_{t}(x-\alpha \sigma_t(x) \xi) \cdot \sigma_{t} (x) \xi\big],
    \end{aligned}
\end{equation}
whose limits as $\alpha\to0$ are $\nabla\cdot s_t(x)$ and $ \tr\left(D_t(x) \nabla s_t(x)\right)$, respectively.
In practice, we observe that this approach always helps stabilize training. Moreover, we observe that use of the denoising loss also stabilizes training, so that it is preferable to full computation of $\nabla \cdot s_t(x)$ even when the latter is feasible.

\section{Gaussian case}
\label{app:gauss}
Here, we consider the case of an Ornstein-Uhlenbeck (OU) process where the score can be written analytically, thereby providing a benchmark for our approach. The example treated in Section~\ref{sec:harmonic} with details in Appendix~\ref{app:exp:harmonic} is a special case of such an OU process with additional symmetry arising from permutations of the particles. The SDE reads
\begin{equation}
    \label{eq:OU:a}
    dX_t = - \Gamma_t (X_t - b_t) dt + \sqrt{2} \sigma_t dW_t
\end{equation}
where $X_t\in\R^d$, $\Gamma_t\in \R^{d\times d}$ is a time-dependent positive-definite tensor (not necessarily symmetric), $b_t\in \R^d$ is a time-dependent vector, and $\sigma_t\in \R^{d\times d}$ is a time-dependent tensor. The Fokker-Planck equation associated with~\eqref{eq:OU:a} is
\begin{equation}
    \label{eq:fpe:OU:a}
    \partial_t \rho_t^*(x) = - \nabla \cdot \left( (\Gamma_t x - b_t) \rho_t^*(x) - D_t \nabla \rho_t^*(x) \right) 
\end{equation}
where $D_t = \sigma_t \sigma_t^\T$. Assuming that the initial condition is Gaussian, $\rho_0 = \mathsf{N}(m_0,C_0)$ with $C_0= C_0^\T\in \R^{d\times d} $ positive-definite, the solution is Gaussian at all times $t\ge0$, $\rho_t^* = \mathsf{N}(m_t,C_t)$ with $m_t$ and $C_t= C_t^\T$ solutions to
\begin{equation}
    \label{eq:m:C:a}
    \begin{aligned}
    \dot m_t & = - \Gamma_t (m_t - b_t)\\
    \dot C_t & = - \Gamma_t C_t - C_t \Gamma_t^\T + 2 D_t
    \end{aligned}
\end{equation}
This implies in particular that
\begin{equation}
    \label{eq:logrhot:a}
    \nabla \log \rho_t^*(x)  = - C_t^{-1} (x- m_t).
\end{equation}
so that the probability flow equation for $X_t$ and the equation for $G_t$ written in~\eqref{eq:sbtm2} read
\begin{equation}
    \label{eq:Xt:Gt:a}
    \begin{aligned}
    \dot X_t (x)& =  (D_t C_t^{-1}-\Gamma_t) X_t(x) + \Gamma_t b_t - D_t C_t^{-1} m_t,\\
    \dot G_t(x) & = (\Gamma_t^\T - C_t^{-1} D_t) G_t(x),
    \end{aligned}
\end{equation}
with initial condition $X_0(x) = x$ and $G_0(x) = \nabla \log \rho_0(x) = -C_0^{-1} (x-m_0)$.
It is easy to see that with $x\sim \rho_0=\mathsf{N}(m_0,C_0)$ we have $X_t(x)\sim \rho_t^* = \mathsf{N}(m_t,C_t)$ since, from the first equation in~\eqref{eq:Xt:Gt:a},  the mean and variance of $X_t$ satisfy~\eqref{eq:m:C:a}. Similarly, when $x\sim \rho_0=\mathsf{N}(m_0,C_0)$,  $G_0(x) \sim N(0,C_0^{-1})$, so that $G_t(x) \sim \mathsf{N}(0,C_t^{-1})$ because the second equation in~\eqref{eq:Xt:Gt:a} is linear and hence preserves Gaussianity. Moreover, $\E_0 G_t(x)=0$ and $B_t = B_t^\T = \E_0 [G_t(x) G_t^\T(x)]$  satisfies
\begin{equation}
\label{eq:B:a}
    \frac{d}{dt} B_t = (\Gamma_t^\T - C_t^{-1} D_t) B_t + B_t (\Gamma_t - D_t C_t^{-1})
\end{equation}
The solution to this equation is $B_t=C_t^{-1}$ since substituting this ansatz into \eqref{eq:B:a} gives the equation for $C_t^{-1}$ that we can deduce from~\eqref{eq:m:C:a}
\begin{equation}
    \label{eq:Cm1:a}
    \frac{d}{dt} C^{-1}_t = C_t^{-1} \dot C_t C_t^{-1} = - C_t^{-1}\Gamma_t  - \Gamma_t^\T C_t^{-1} + 2 C_t^{-1} D_t C_t^{-1}.
\end{equation}

Note that if $\Gamma_t=\Gamma$, $b_t = b$, and $D_t=D$ are all time-independent, then $\lim_{t\to\infty} \rho_t = N(m_\infty,C_\infty)$ with $m_\infty = b$ and $C_\infty$ the solution to the Lyapunov matrix equation
\begin{equation}
    \label{eq:Cinfty:a}
    \Gamma C_\infty + C_\infty \Gamma^\T = 2 D.
\end{equation}
This means that at long times the coefficients at the right-hand sides of \eqref{eq:Xt:Gt:a} also settle on constant values. However, $X_t$ and $G_t$ do not necessarily stop evolving; one situation where they too converge is when the OU process is in detailed balance, i.e. when $\Gamma = D A$ for some $A=A^\T \in \R^{d\times d} $ positive-definite. In that case,  the solution to~\eqref{eq:Cinfty:a} is $C_\infty = A^{-1}$ and it is easy to see that at long times the right-hand sides of \eqref{eq:Xt:Gt:a} tend to zero. 

\begin{remark}
This last conclusion is actually more generic than for a simple OU process. For any SDE in detailed balance, i.e. that can be written as
\begin{equation}
    \label{eq:detailed}
    dX_t = - D(X_t) \nabla U(X_t) dt + \nabla \cdot D(X_t) dt + \sqrt{2} \sigma_t(X_t) dW_t
\end{equation}
where $U:\R^d\to \R_{>0}$ is a $C^2$-potential such that $Z= \int_{\R^d} e^{-U(x)} dx <\infty$, we have that $\lim_{t\to\infty} \rho_t (x) = Z^{-1} e^{-U(x)}$, and the corresponding flows $X_t$ and $G_t$ eventually stop as $t\to\infty$. In this case, $\rho_t$ follows gradient descent in $W_2$ over the energy
\begin{equation}
    \label{eq:energy}
    E[\rho] = \int_{\R^d} (U(x) + \log \rho(x) ) \rho(x) dx
\end{equation}
The unique PDF minimizing this energy is $Z^{-1} e^{-U(x)}$, and as $t\to\infty$ $X_t$ converges towards a transport map between the initial $\rho_0$ and  $Z^{-1} e^{-U(x)}$.
\end{remark}

\section{Experimental details and additional examples}
\label{app:exp}
All numerical experiments were performed in $\texttt{jax}$ using the $\texttt{dm-haiku}$ package to implement the networks and the $\texttt{optax}$ package for optimization. 

\subsection{Harmonically interacting particles in a harmonic trap}
\label{app:exp:harmonic}
\paragraph{Network architecture} Both the single-particle energy $U_{\theta_t, 1}: \R^d\rightarrow \R$ and two-particle interaction energy $U_{\theta_t, 2}: \R^d\times\R^d \rightarrow \R$ are parameterized as single hidden-layer neural networks with the \texttt{swish} activation function~\citep{swish_act} and $\mathtt{n\_hidden} = 100$ hidden neurons. The hidden layer biases are initialized to zero while the hidden layer weights are initialized from a truncated normal distribution with variance $1/\mathtt{fan\_in}$, following the guidelines recommended in~\citep{ioffe_batch_2015}.

\paragraph{Optimization} The Adam~\citep{kingma_adam_2017} optimizer is used with an initial learning rate of $\eta = 10^{-4}$ and otherwise default settings. At time $t=0$, the analytical relative loss
\begin{equation}
    \label{eqn:loss_analytical}
    L[s_0] = \frac{\int |s_0(x) - \nabla \log \rho_0(x)|^2\rho_0(x) dx}{\int |\nabla \log \rho_0(x)|^2\rho_0(x) dx}
\end{equation}
is minimized to a value less than $10^{-4}$ using knowledge of the initial condition $\rho_0 = \mathsf{N}\left(\beta_0, \sigma_0^2I\right)$ with $\sigma_0 = 0.25$. In \eqref{eqn:loss_analytical}, the expectation with respect to $\rho_0$ is approximated by an initial set of samples $x_{j} = \left(x^{(1)}_{j}, x^{(2)}_{j}, \hdots, x^{(N)}_{j}\right)^\T$ with $j = 1, \hdots, n$ drawn from $\rho_0$. In the experiments, we set $n = 100$, which we found to be sufficient to obtain a few digits of relative accuracy on various quantities of interest. We set the physical timestep $\Delta t = 10^{-3}$ and take $\mathtt{n\_opt\_steps} = 25$ steps of Adam until the norm of the gradient is below $\mathtt{gtol} = 0.1$.

\paragraph{Analytical moments} First define the mean, second moment, and covariance according to
\begin{align*}
    m_t^{(i)} &= \E\big[X_t^{(i)}\big],\\
    M_t^{(ij)} &= \E\big[X_t^{(i)} \big(X_t^{(j)}\big)^\T\big],\\
    C_t^{(ij)} &= M^{(ij)} - m^{(i)}\big(m^{(j)}\big)^\T.
\end{align*}
It is straightforward to show that the mean and covariance obey the dynamics
\begin{align}
    \label{eqn:harmonic_mean_gen}
    \dot{m}_t^{(i)} &= -(m_t^{(i)} - \beta_t) +\frac{\alpha}{N}\sum_{k=1}^{N}\left(m_t^{(i)} - m_t^{(k)}\right),\\
    \label{eqn:harmonic_cov_gen}
    \dot{C}_t^{(ij)} &= -2 (1 - \alpha)C_t^{(ij)} + 2DI\delta_{ij} -       \frac{\alpha}{N}\sum_{k=1}^{N} \left(C_t^{(kj)} + C_t^{(ik)}\right)
\end{align}
Because the particles are indistinguishable so long as they are initialized from a distribution that is symmetric with respect to permutations of their labeling, the moments will satisfy the ansatz
\begin{align}
    \label{eqn:symmetric_mean}
    m_t^{(i)} &= \bar m(t), \:\:\: i=1, \hdots, N\\
    \label{eqn:symmetric_cov}
    C_t^{(ij)} &= C_d(t) \delta_{ij} + C_o(t) (1 - \delta_{ij}), \:\:\: i, j=1, \hdots, N.
\end{align}
The dynamics for the vector $\bar m: \R_{\geq 0} \rightarrow \R^{\bar{d}}$, as well as the matrices $C_d: \R_{\geq 0} \to \R^{\bar{d} \times \bar{d}}$ and $C_o: \R_{\geq 0} \to \R^{\bar{d} \times \bar{d}}$ can then be obtained from \eqref{eqn:harmonic_mean_gen} and \eqref{eqn:harmonic_cov_gen} as
\begin{align*}
    \dot{\bar m} &= \beta_t - \bar m,\\
    \dot{C}_d &= 2(\alpha - 1) C_d - 2\frac{\alpha}{N}\left(C_d + (n-1)C_o\right) + 2DI,\\
    \dot{C}_o &= 2(\alpha - 1)C_o - 2\frac{\alpha}{N}\left(C_d + (n-1)C_o\right).
\end{align*}
For a given $\beta:\R\rightarrow\R^{\bar{d}}$, these equations can be solved analytically in Mathematica as a function of time, giving the mean ${m}_t = \bar m(t)\otimes 1_N \in \R^{N\bar{d}}$ and covariance $C_t = (C_d(t) - C_o(t)) \otimes I_{N\times N} + C_o(t) \otimes \left(1_N 1_N^\T\right)\in \R^{N\bar{d}\times N\bar{d}}$. Because the solution is Gaussian for all $t$, we then obtain the analytical solution to the Fokker-Planck equation $\rho^*_t = \mathsf{N}\left(m _t, C_t\right)$ and the corresponding analytical score $-\nabla \log \rho^*_t(x) = C_t^{-1}(x -  m_t)$.

\paragraph{Potential structure}
Here, we show that the potential for this example lies in the class of potentials described by~\eqref{eqn:parametric_potential}. From Equation~\ref{eqn:symmetric_cov}, we have a characterization of the structure of the covariance matrix $C_t$ for the analytical potential $U_t(x) = \frac{1}{2}(x - m_t)^\T C_t^{-1} (x - m_t)$. In particular, $C_t$ is block circulant, and hence is block diagonalized by the roots of unity (the block discrete Fourier transform). That is, we may take a ``block eigenvector'' of the form $\omega_k = \left(I_{\bar{d}\times \bar{d}} \rho^k, I_{\bar{d}\times \bar{d}} \rho^{2k}, \hdots, I_{\bar{d}\times \bar{d}} \rho^{(N-1)k}\right)^\T$ with $\rho = \exp(-2\pi i /N)$ for $k = 0, \hdots N-1$. By direct calculation, this block diagonalization leads to two distinct block eigenmatrices, \begin{equation*}
    C_t = V \begin{pmatrix} C_d(t) + (N-1)C_o(t) & 0 & 0 & \hdots & 0\\
    0 & C_d(t) - C_o(t) & 0 & \hdots & 0\\
    0 & 0 & \ddots & \hdots & 0 \\
    0 & 0 & 0 & \hdots & C_d(t) - C_o(t) \end{pmatrix} V^{-1}
\end{equation*}
where $V \in \R^{N\bar{d} \times N\bar{d}}$ denotes the matrix with block columns $\omega_k$. The inverse matrix $C_t^{-1}$ then must similarly have only two distinct block eigenmatrices given by $\left(C_d(t) + (N-1)C_o(t)\right)^{-1}$ and $\left(C_d(t) - C_o(t)\right)^{-1}$. By inversion of the block Fourier transform, we then find that 
\begin{equation*}
    \left(C_t^{-1}\right)^{(ij)} = \bar{C}_d\delta_{ij} + \bar{C}_o(1 - \delta_{ij})   
\end{equation*}
for some matrices $\bar{C}_d, \bar{C}_o$. Hence, by direct calculation
\begin{align*}
    \left(x - m_t\right)^\T C_t^{-1} \left(x - m_t\right) &= \sum_{i, j}^N \left(x^{(i)} - m_t^{(i)}\right)^\T\left(C_t^{-1}\right)^{(ij)}\left(x^{(j)} - m_t^{(j)}\right)\\
    &= \sum_{i, j}^N \left(x^{(i)} - \bar m(t)\right)^\T\left(\bar{C}_d\delta_{ij} + \bar{C}_o (1 - \delta_{ij})\right)\left(x^{(j)} - \bar m(t)\right)\\
    &= \sum_{i}^N \left(x^{(i)} - \bar m(t)\right)^\T\bar{C}_d\left(x^{(i)} - \bar m(t)\right)^\T \\
    &\qquad + \sum_{i\neq j}^N\left(x^{(i)} - \bar m(t)\right)^\T\bar{C}_o\left(x^{(j)} - \bar m(t)\right)
\end{align*}
Above, we may identify the first term in the last line as $\sum_{i=1}^N U_1(x^{(i)})$ and the second term in the last line as $\frac{1}{N}\sum_{i \neq j}^N U_2(x^{(i)}, x^{(j)})$. Moreover, $U_2(\cdot, \cdot)$ is symmetric with respect to its arguments.

\paragraph{Analytical Entropy} For this example, the entropy can be computed analytically and compared directly to the learned numerical estimate. By definition,
\begin{align*}
    s_t &= -\int_{\mathbb{R}^{N\bar d}} \log \rho_t(x) \rho_t(x)dx,\\
    &= - \int_{\mathbb{R}^{N\bar d}}  \left(-\frac{N\bar d}{2}\log(2\pi) - \frac{1}{2}\log\det C_t - \frac{1}{2} (x-m_t)^\T C_t^{-1}(x-m_t)\right)\rho_t(x)dx,\\
    &= \frac{N\bar d}{2}\left(\log\left(2\pi\right) + 1\right) + \frac{1}{2}\log\det C_t.
\end{align*}

\paragraph{Additional figures}
Images of the learned velocity field and potential in comparison to the corresponding analytical solutions can be found in Figures~\ref{fig:harmonic_velocity} and~\ref{fig:harmonic_potential}, respectively. Further detail can be found in the corresponding captions. We stress that the two-dimensional images represent single-particle slices of the high-dimensional functions.
\begin{figure}
    \centering
    \includegraphics[width=\textwidth]{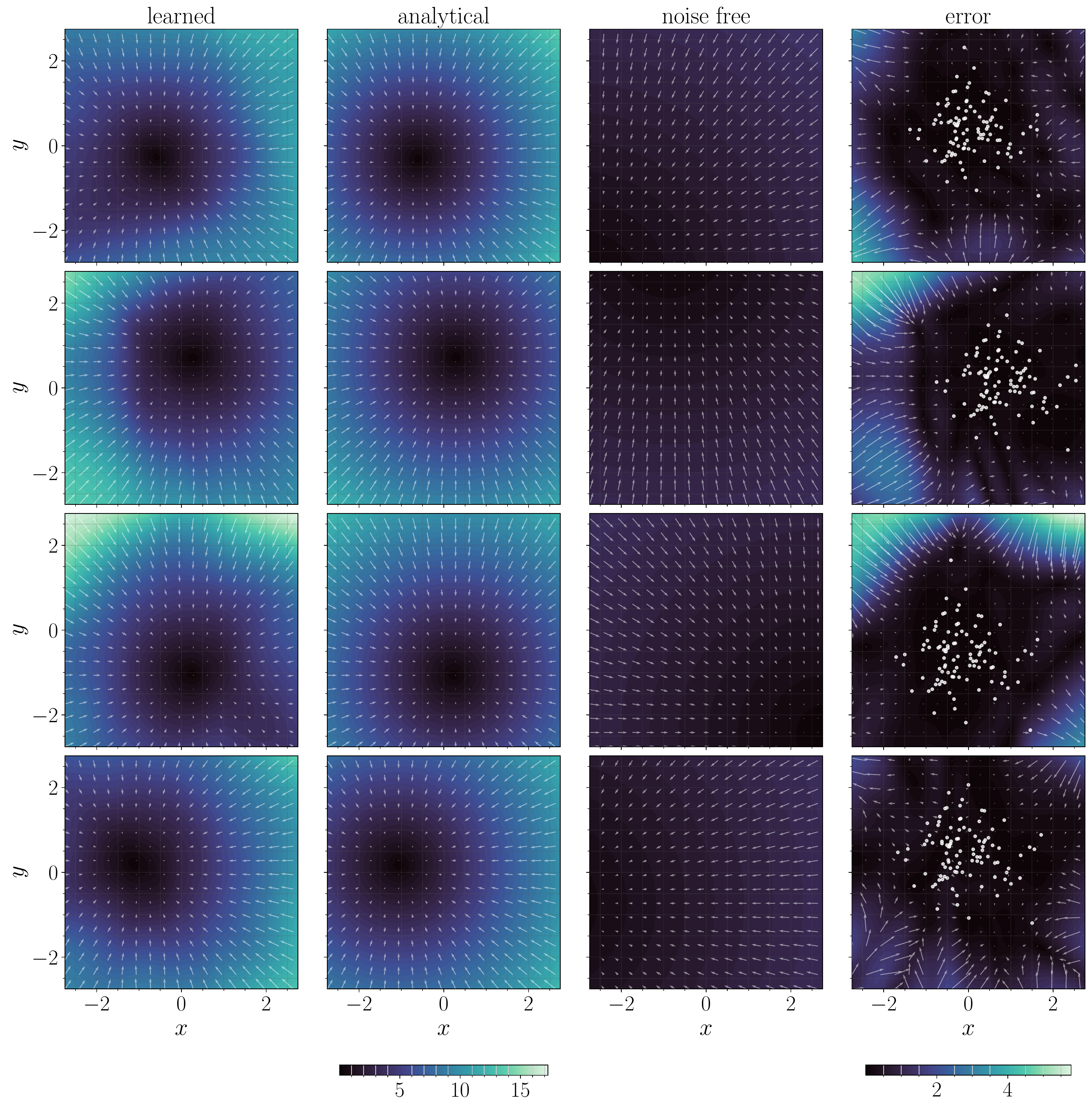}
    \caption{\textit{A system of $N=50$ harmonically interacting particles in a harmonic trap: slices of the high-dimensional velocity field.} Cross sections of the velocity field for $N=50$ harmonically interacting particles in a moving harmonic trap. Columns depict the learned, analytical, noise-free, and error between the learned and analytical velocity fields, respectively. Rows indicate different time points, corresponding to $t = 1.25, 2.5, 3.75, $ and $5.0$, respectively. Each velocity field is plotted as a function of a single particle's coordinate (denoted as $x$ and $y$); all other particle coordinates are fixed to be at the location of a sample. Color depicts the magnitude of the velocity field while arrows indicate the direction. Learned, analytical, and noise-free share a colorbar for direct comparison; the error occurs on a different scale and is plotted with its own colorbar. White circles in the error plot indicate samples projected onto the $xy$ plane; locations of low error correlate well with the presence of samples.}
    \label{fig:harmonic_velocity}
\end{figure}

\begin{figure}
    \centering
    \includegraphics[width=0.975\textwidth]{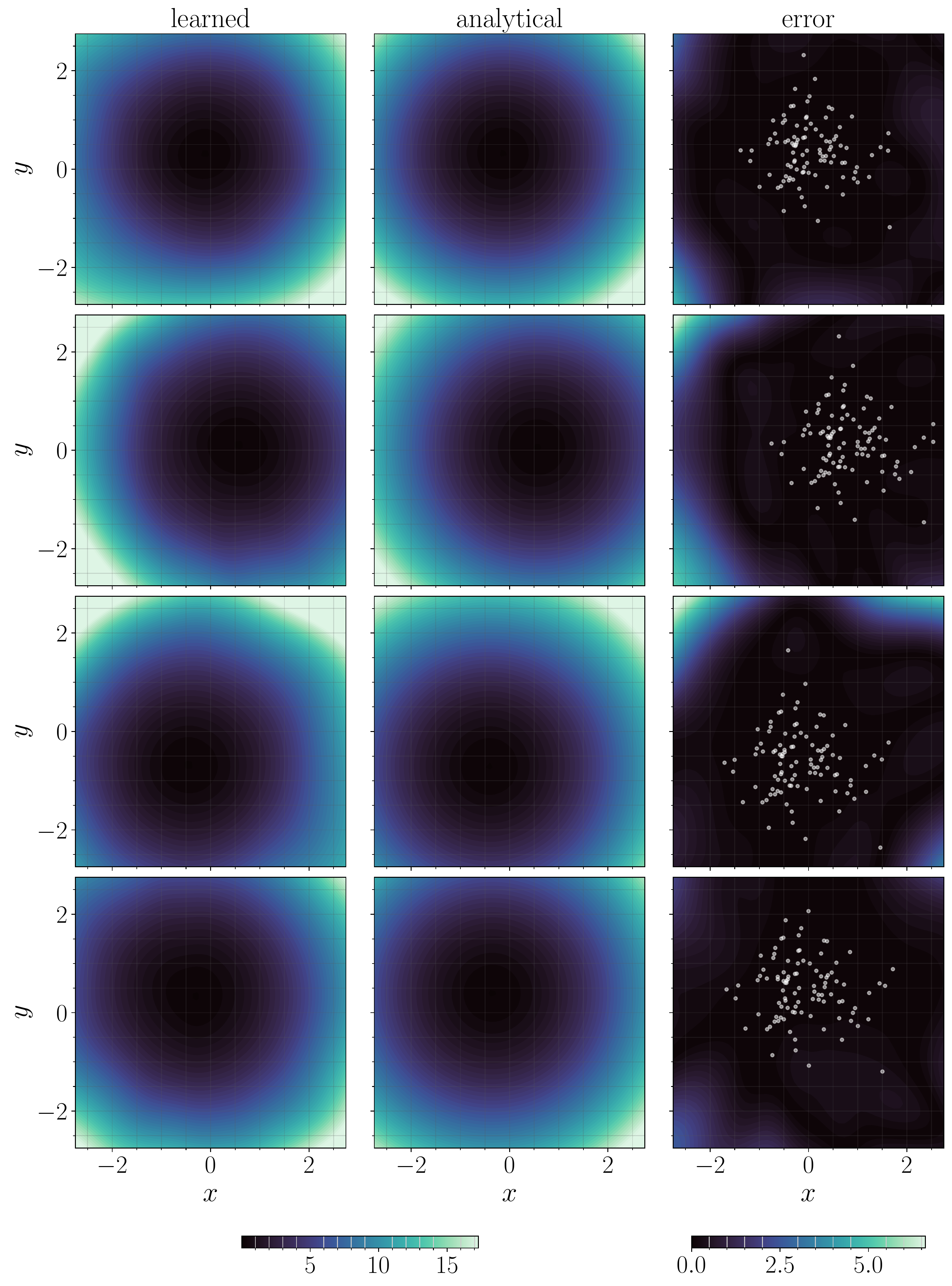}
    \caption{\textit{A system of $N=50$ harmonically interacting particles in a harmonic trap: slices of the high-dimensional potential.} Cross sections of the potential field $U_{\theta_t}(x)$ computed via~\eqref{eqn:parametric_potential}.
    Columns depict the learned, analytical, and error between the learned and analytical, respectively. 
    Rows indicate distinct time points, corresponding to $t = 1.25, 2.5, 3.75,$ and $5.0$, respectively. 
    As in Figure~\ref{fig:harmonic_velocity}, each potential field is plotted as a function of a single particle's coordinate (denoted as $x$ and $y$) with other particle coordinates fixed on a sample.
    All potentials are normalized via an overall shift so that the minimum value is zero.
    White circles in the error plot indicate samples from the learned system projected onto the $xy$ plane.}
    \label{fig:harmonic_potential}
\end{figure}

\subsection{Soft spheres in an anharmonic trap}
\label{app:exp:anharmonic}
\paragraph{Network architecture} Both potential terms $U_{\theta_t, 1}$ and $U_{\theta_t, 2}$ are modeled as four hidden-layer deep fully connected networks with $\mathtt{n\_hidden} = 32$ neurons in each layer. The initialization is identical to Appendix~\ref{app:exp:anharmonic}.

\paragraph{Optimization and initialization} The Adam optimizer is used with an initial learning rate of $\eta = 5\times 10^{-3}$ and otherwise default settings. At time $t=0$, the loss \eqref{eqn:loss_analytical} is minimized to a value less than $10^{-6}$ over $n$ samples $X_0^{(i)} \sim \otimes_{j=1}^N \mathsf{N}(\beta_0, \sigma_0^2I)$, $i=1, \hdots, n$ with $\sigma_0 = 0.5$ and $n=10^4$. Past this initial stage, the denoising loss is used with a noise scale $\sigma = 0.1$; we found that a higher noise scale regularized the problem and led to a smoother prediction for the entropy, at the expense of a slight bias in the moments. By increasing the number of samples $n$, the noise scale can be reduced while maintaining an accurate prediction for the entropy. The loss is minimized by taking $\mathtt{n\_opt\_steps}= 25$ steps of Adam at each timestep. The physical timestep is set to $\Delta t = 10^{-3}$.

\paragraph{Additional figures} Figures~\ref{fig:anharmonic_circular_moments} and~\ref{fig:anharmonic_linear_moments} show the full grid of covariance components for the SDE, learned, and noise free systems. The noise free underestimates the moments, while the learned and SDE agree well.

\begin{figure}
    \centering
    \includegraphics[width=\textwidth]{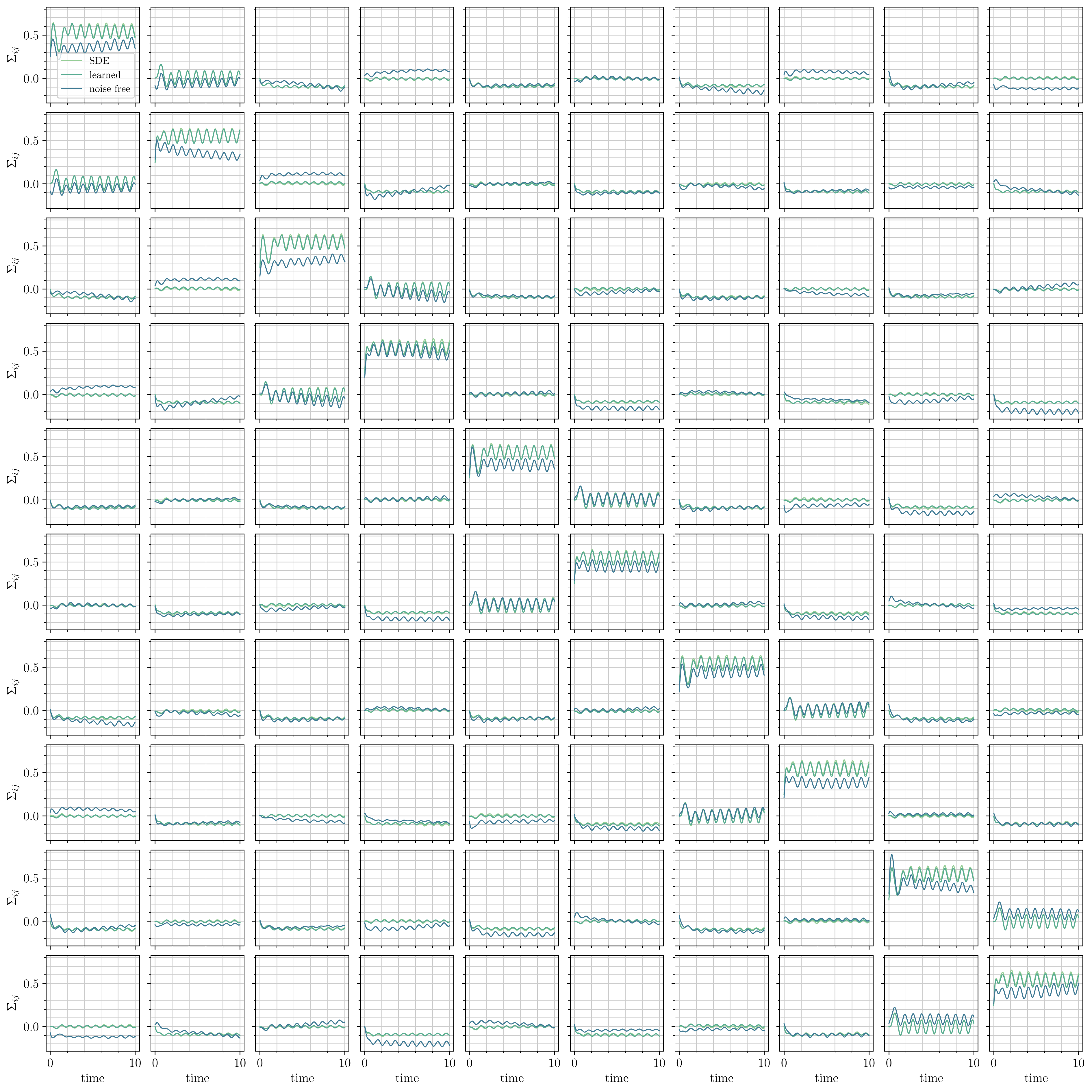}
    \caption{\textit{A system of $N=5$ soft-sphere particles in an anharmonic trap: moments.} All components of the covariance matrix over time for the circular trap motion. The learned system and the stochastic system agree well, while the noise free system underestimates the moments.
    }
    \label{fig:anharmonic_circular_moments}
\end{figure}
\begin{figure}
    \centering
    \includegraphics[width=\textwidth]{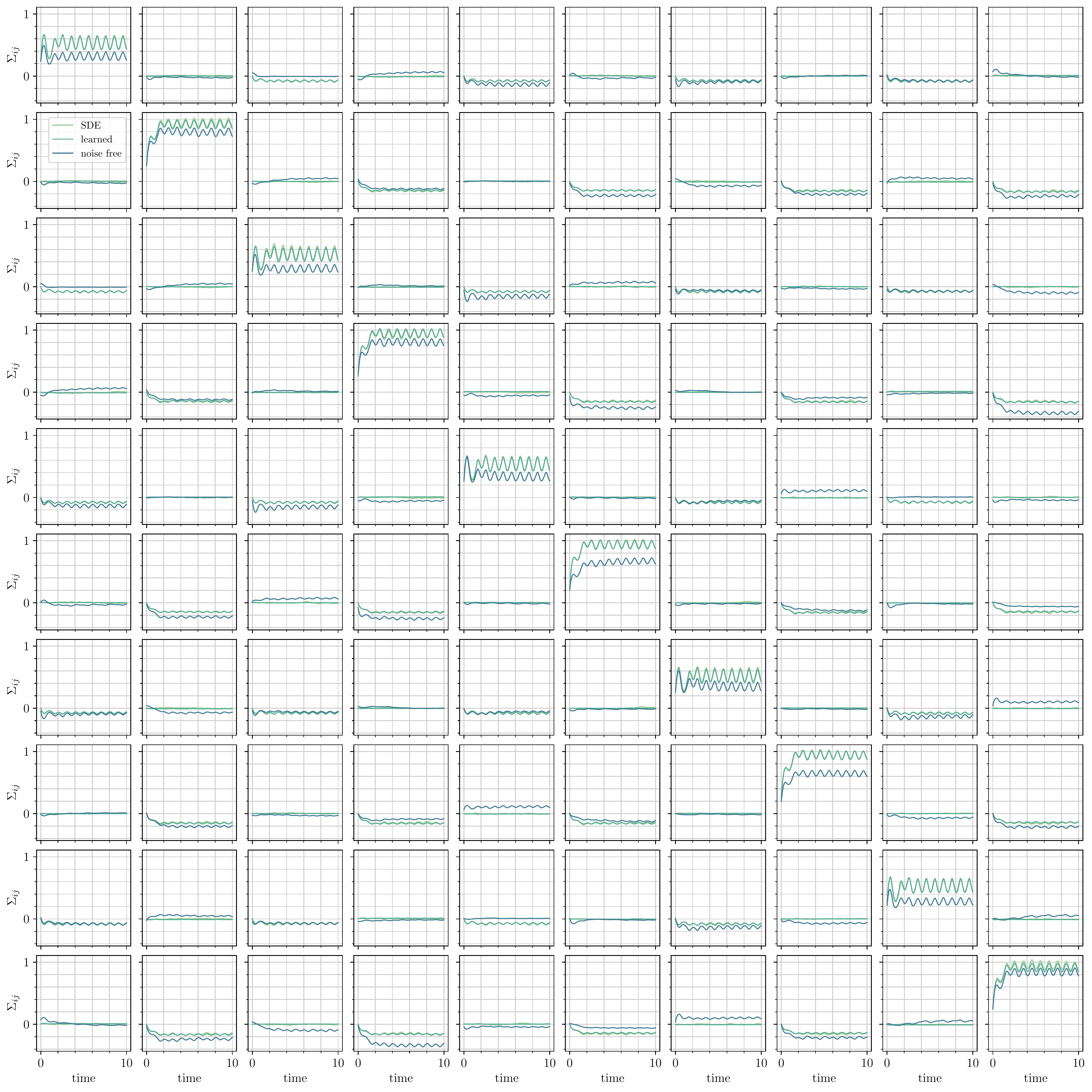}
    \caption{\textit{A system of $N=5$ soft-sphere particles in an anharmonic trap: moments.} All components of the covariance matrix over time for the linear trap motion. The learned system and the stochastic system agree well, while the noise free system underestimates the moments.
    }
    \label{fig:anharmonic_linear_moments}
\end{figure}

\subsection{An active swimmer}
\label{app:swimmer}
\paragraph{Setup} We parameterize the score directly $s_t: \R^{2}\rightarrow\R$ using a three hidden layer neural network with $\mathtt{n\_hidden} = 32$ neurons per hidden layer. Because the dynamics is anti-symmetric, we impose that $s(x, v) = -s(-x, -v)$.

\paragraph{Optimization and initialization} The network initialization is identical to the previous two experiments. The physical timestep is set to $\Delta t = 10^{-3}$. The Adam optimizer is used with an initial learning rate of $\eta = 10^{-4}$. At time $t=0$ the loss \eqref{eqn:loss_analytical} is minimized to a tolerance of $10^{-4}$ over $n=10^4$ samples drawn from an initial distribution $\mathsf{N}(0, \sigma_0^2I)$ with $\sigma_0 = 1$. The denoising loss is used with a noise scale $\sigma = 0.05$, using $\mathtt{n\_opt\_steps} = 25$ steps of Adam until the norm of the gradient is below $\mathtt{gtol} = 0.5$.

\end{document}